\documentclass[draftcls,onecolumn]{IEEEtran}
\usepackage{xr}
\usepackage{cite}
\usepackage{amsmath}
\usepackage[colorlinks=true,linkcolor=black,citecolor=black,urlcolor=black]{hyperref}
\usepackage{cases}
\usepackage[utf8]{inputenc}
\usepackage[english]{babel}
\usepackage{footnote}
\usepackage{booktabs}

\usepackage{bm}

\usepackage[mathscr]{eucal}
\usepackage{epsfig,epsf,psfrag}
\usepackage{amssymb,amsmath,amsfonts,latexsym}
\usepackage{amsmath,graphicx,bm,xcolor,url}
\usepackage[caption=false]{subfig} 
\usepackage{fixltx2e}
\usepackage{array}
\usepackage{verbatim}
\usepackage{bm}
\usepackage{algpseudocode}
\usepackage{algorithm}
\usepackage{verbatim}
\usepackage{textcomp}
\usepackage{mathrsfs}
\usepackage{epstopdf}
\usepackage{relsize}
\usepackage{cleveref} 
\usepackage{subfig}
 \usepackage{amsthm}

\catcode`~=11 \def\UrlSpecials{\do\~{\kern -.15em\lower .7ex\hbox{~}\kern .04em}} \catcode`~=13 

\allowdisplaybreaks[3]

\newcommand{\calD}{\mathcal{D}}

\newcommand{\calK}{\mathcal{K}}
\newcommand{\calL}{\mathcal{L}}

\newcommand{\calN}{\mathcal{N}}

\newcommand{\calP}{\mathcal{P}}

\newcommand{\ba}{\mathbf{a}}
\newcommand{\bA}{\mathbf{A}}

\newcommand{\bB}{\mathbf{B}}

\newcommand{\bI}{\mathbf{I}}

\newcommand{\bs}{\mathbf{s}}

\newcommand{\bu}{\mathbf{u}}

\newcommand{\bv}{\mathbf{v}}

\newcommand{\bW}{\mathbf{W}}
\newcommand{\bx}{\mathbf{x}}
\newcommand{\bX}{\mathbf{X}}
\newcommand{\by}{\mathbf{y}}

\newcommand{\bz}{\mathbf{z}}

\newcommand{\bbE}{\mathbb{E}}

\newcommand{\bbN}{\mathbb{N}}

\newcommand{\bbR}{\mathbb{R}}

\DeclareMathAlphabet{\mathbsf}{OT1}{cmss}{bx}{n}
\DeclareMathAlphabet{\mathssf}{OT1}{cmss}{m}{sl}

\DeclareSymbolFont{bsfletters}{OT1}{cmss}{bx}{n}  
\DeclareSymbolFont{ssfletters}{OT1}{cmss}{m}{n}
\DeclareMathSymbol{\bsfGamma}{0}{bsfletters}{'000}
\DeclareMathSymbol{\ssfGamma}{0}{ssfletters}{'000}
\DeclareMathSymbol{\bsfDelta}{0}{bsfletters}{'001}
\DeclareMathSymbol{\ssfDelta}{0}{ssfletters}{'001}
\DeclareMathSymbol{\bsfTheta}{0}{bsfletters}{'002}
\DeclareMathSymbol{\ssfTheta}{0}{ssfletters}{'002}
\DeclareMathSymbol{\bsfLambda}{0}{bsfletters}{'003}
\DeclareMathSymbol{\ssfLambda}{0}{ssfletters}{'003}
\DeclareMathSymbol{\bsfXi}{0}{bsfletters}{'004}
\DeclareMathSymbol{\ssfXi}{0}{ssfletters}{'004}
\DeclareMathSymbol{\bsfPi}{0}{bsfletters}{'005}
\DeclareMathSymbol{\ssfPi}{0}{ssfletters}{'005}
\DeclareMathSymbol{\bsfSigma}{0}{bsfletters}{'006}
\DeclareMathSymbol{\ssfSigma}{0}{ssfletters}{'006}
\DeclareMathSymbol{\bsfUpsilon}{0}{bsfletters}{'007}
\DeclareMathSymbol{\ssfUpsilon}{0}{ssfletters}{'007}
\DeclareMathSymbol{\bsfPhi}{0}{bsfletters}{'010}
\DeclareMathSymbol{\ssfPhi}{0}{ssfletters}{'010}
\DeclareMathSymbol{\bsfPsi}{0}{bsfletters}{'011}
\DeclareMathSymbol{\ssfPsi}{0}{ssfletters}{'011}
\DeclareMathSymbol{\bsfOmega}{0}{bsfletters}{'012}
\DeclareMathSymbol{\ssfOmega}{0}{ssfletters}{'012}

\newcommand{\balpha}{\bm{\alpha}}

\newcommand{\bmeta}{\bm{\eta}}

\newtheorem{theorem}{Theorem} 
\newtheorem{lemma}{Lemma}

\newtheorem{definition}{Definition} 
\newtheorem{remark}{Remark}

\newcommand{\qednew}{\nobreak \ifvmode \relax \else
      \ifdim\lastskip<1.5em \hskip-\lastskip
      \hskip1.5em plus0em minus0.5em \fi \nobreak
      \vrule height0.75em width0.5em depth0.25em\fi}

\begin{document}
    
\title{Projected Gradient Descent Algorithms for \\Solving Nonlinear Inverse Problems with Generative Priors}

\author{Zhaoqiang Liu, Jun Han

\thanks{
Z.~Liu is with the Department of Computer Science, National University of Singapore (email: \url{dcslizha@nus.edu.sg}). 

J.~Han is with Platform and Content Group, Tencent (email: \url{junhanjh@tencent.com}).}}

\maketitle

\begin{abstract}
    In this paper, we propose projected gradient descent (PGD) algorithms for signal estimation from noisy nonlinear measurements. We assume that the unknown $p$-dimensional signal lies near the range of an $L$-Lipschitz continuous generative model with bounded $k$-dimensional inputs. In particular, we consider two cases when the nonlinear link function is either unknown or known. For unknown nonlinearity, similarly to~\cite{liu2020generalized}, we make the assumption of sub-Gaussian observations and propose a linear least-squares estimator. We show that when there is no representation error and the sensing vectors are Gaussian, roughly $O(k \log L)$ samples suffice to ensure that a PGD algorithm converges linearly to a point achieving the optimal statistical rate using arbitrary initialization. For known nonlinearity, we assume monotonicity as in~\cite{yang2016sparse}, and make much weaker assumptions on the sensing vectors and allow for representation error. We propose a nonlinear least-squares estimator that is guaranteed to enjoy an optimal statistical rate. A corresponding PGD algorithm is provided and is shown to also converge linearly to the estimator using arbitrary initialization. In addition, we present experimental results on image datasets to demonstrate the performance of our PGD algorithms.
\end{abstract}

\section{Introduction}\label{sec:intro}

Over the past two decades, the theoretical and algorithmic aspects of high-dimensional {\em linear} inverse problems have been studied extensively. The standard compressed sensing (CS) problem, which models low-dimensional structure via the sparsity assumption, 
is particularly well-understood~\cite{Fou13}. 

Despite the popularity of linear CS, in many real-world applications, nonlinearities may arise naturally, and it is more desirable to adopt {\em nonlinear} measurement models. For example, the semi-parametric {\em single index model (SIM)}, which is formulated below, is a popular nonlinear measurement model that has long been studied~\cite{han1987non}:
\begin{equation}\label{eq:sim_first}
 y_i = f_i(\ba_i^T\bx^*), \quad i = 1,2,\ldots,n, 
\end{equation}
where $\bx^* \in \bbR^p$ is an unknown signal that is close to some structured set $\calK$, $\ba_i \in \bbR^p$ are the sensing vectors, and $f_i \,:\, \bbR \rightarrow \bbR$ are i.i.d. realizations of an {\em unknown (possibly random)} function $f$. In general, $f$ plays the role of a nonlinearity, and it is called a {\em link} function. The goal is to estimate $\bx^*$ despite this unknown link function. Note that since the norm of $\bx^*$ can be absorbed into the unknown $f$, the signal $\bx^*$ is typically assumed to have unit $\ell_2$-norm.  

In addition, inspired by the tremendous success of deep generative models in numerous real-world applications, recently, for the CS problem, it has been of interest to replace the sparsity assumption with the generative model assumption. More specifically, instead of being assumed to be sparse, the signal is assumed to lie near the range of a generative model, typically corresponding to a deep neural network~\cite{bora2017compressed}. Along with several theoretical developments, the authors of~\cite{bora2017compressed} perform extensive numerical experiments on image datasets to demonstrate that for a given accuracy, generative priors can reduce the required number of measurements by a factor of $5$ to $10$. There are a variety of follow-up works of~\cite{bora2017compressed}, including~\cite{heckel2019deep,ongie2020deep,jalal2021instance}, among others. 

In this paper, following the developments in both sparsity-based nonlinear inverse problems and inverse problems with generative priors, we provide theoretical guarantees for projected gradient descent (PGD) algorithms devised for nonlinear inverse problems using generative models. 

\subsection{Related Work}

The most relevant existing works can roughly be divided into (i) nonlinear inverse problems without generative priors, and (ii) inverse problems with generative priors.

{\bf Nonlinear inverse problems without generative priors:} The SIM has long been studied in the low-dimensional setting where $p \ll n$, based on various assumptions on the sensing vector or link function. For example, the maximum rank correlation estimator has been proposed in~\cite{han1987non} under the assumption of a monotonic link function. In recent years, the SIM has also been studied in \cite{plan2016generalized,genzel2016high,plan2017high,oymak2017fast} in the high-dimensional setting where an accurate estimate can be obtained when $n \ll p$, with the sensing vectors being assumed to be Gaussian. In particular, the authors of~\cite{plan2016generalized} show that the generalized Lasso approach works for high-dimensional SIM under the assumption that the set of structured signal $\calK$ is convex, which is in general not satisfied for the range of a generative model with the Lipschitz continuity. 

Nonetheless, the generality of the {\em unknown} link function in SIM comes at a price. Specifically, as mentioned above, it is necessary for the works studying SIM to assume the distribution of the sensing vector to be Gaussian or symmetric elliptical, and for nonlinear signal estimation problems with general sensing vectors, in order to achieve consistent estimation, knowledge of the link function is required~\cite{zhang2018nonlinear}. In addition, when the link function is unknown, since the norm of the signal may be absorbed into this link function, there is an identifiability issue and we are only able to estimate the direction of the signal. In practice, this can be unsatisfactory and may lead to large estimation errors. Moreover, for an unknown nonlinearity, it remains an open problem to handle signals with representation error~\cite{plan2016generalized}, i.e., the signal (up to a fixed scale factor) is not exactly contained in $\calK$. 

Based on these issues of SIM and some applications in machine learning such as the activation functions of deep neural networks~\cite{yang2016sparse}, nonlinear measurement models with {\em known} and {\em monotonic} link functions\footnote{For this case, a natural idea is to apply approaches for linear measurement models to the inverted data $\{f^{-1}(y_i),\ba_i\}_{i=1}^n$. Unfortunately, such a simple idea works well only in the noiseless setting. See~\cite[Section~2]{yang2016sparse} for a discussion.} have been studied in~\cite{yang2016sparse,zhang2018nonlinear,soltani2017fast}. In~\cite{yang2016sparse}, an $\ell_1$-regularized nonlinear least-squares estimator is proposed, and an iterative soft thresholding algorithm is provided to efficiently approximate this estimator. The authors of~\cite{zhang2018nonlinear} propose an iterative hard thresholding (IHT) algorithm to minimize a nonlinear least-squares loss function subject to a combinatorial constraint. In the work~\cite{soltani2017fast}, the demixing problem is formulated as minimizing a special (not the typical least-squares) loss function under a combinatorial constraint, and a corresponding IHT algorithm is designed to approximately find a minimizer. All the algorithms proposed in~\cite{yang2016sparse,zhang2018nonlinear,soltani2017fast} can be regarded as special cases of the PGD algorithm. 

{\bf Inverse problems with generative priors:} 
Bora {\em et al.} show that when the generative model is $L$-Lipschitz continuous with bounded $k$-dimensional inputs, roughly $O(k \log L)$ random Gaussian linear measurements are sufficient to attain accurate estimates~\cite{bora2017compressed}. Their analysis is based on minimizing a linear least-squares loss function, and the objective function is minimized directly over the latent variable in $\bbR^k$ using gradient descent. A PGD algorithm in the ambient space in $\bbR^p$ has been proposed in~\cite{shah2018solving,peng2020solving} for noiseless and noisy Gaussian linear measurements respectively. It has been empirically demonstrated that this PGD algorithm leads to superior reconstruction performance over the algorithm used in~\cite{bora2017compressed}. Various nonlinear measurement models with known nonlinearity have also been studied for generative priors. Specifically, near-optimal sample complexity bounds for $1$-bit measurement models have been presented in~\cite{qiu2020robust,liu2020sample}. Furthermore, the works~\cite{wei2019statistical,liu2020generalized} have provided near-optimal non-uniform recovery guarantees for nonlinear compressed sensing with an unknown nonlinearity. More specifically, the authors of~\cite{wei2019statistical} assume that the link function is differentiable and propose estimators via score functions based on the first and second order Steins identity. The differentiability assumption fails to hold for $1$-bit and other quantized measurement models. To take such measurement models into consideration, the work~\cite{liu2020generalized} instead makes the assumption that the (uncorrupted) observations are sub-Gaussian, and proposes to use a simple linear least-squares estimator despite the unknown nonlinearity. While obtaining these estimators is practically hard due to the typical non-convexity of the range of a generative model, both works are primarily theoretical, and no practical algorithm is provided to approximately find the estimators.

\subsection{Contributions}

Throughout this paper, we make the assumption that the generative model is $L$-Lipschitz continuous with bounded $k$-dimensional inputs (see, e.g.,~\cite{bora2017compressed}). The main contributions of this paper are as follows:
\begin{itemize}
 \item For the scenario of SIM with unknown nonlinearity, we assume that the sensing vector is Gaussian and the signal is exactly contained in the range of the generative model, and propose a PGD algorithm for a linear least-squares estimator. We show that roughly $ O(k\log L)$ samples suffice to ensure that this PGD algorithm converges linearly and yields an estimator with optimal statistical rate, which is roughly of order $\sqrt{(k \log L) / n}$. While this PGD algorithm is identical to the PGD algorithm for solving {\em linear} inverse problems using generative models as proposed in~\cite{shah2018solving,peng2020solving}, the corresponding analysis is significantly different since we consider the SIM with an unknown nonlinear function $f$, instead of the simple linear measurement model. Moreover, unlike~\cite{shah2018solving,peng2020solving}, we have provided a neat theoretical guarantee for choosing the step size. 
 
 \item For the scenario where the link function is known, we make much weaker assumptions for sensing vectors and allow for representation error, i.e., the signal do not quite reside in the range of the generative model, and propose a nonlinear least-squares estimator. We prove that the estimator enjoys optimal statistical rate, and show that a corresponding PGD algorithm converges linearly to this estimator. To the best of our knowledge, the corresponding PGD algorithm ({\em cf.}~\eqref{eq:pgd_known}) is novel. 
 
 \item We perform various numerical experiments on image datasets to back up our theoretical results.
\end{itemize}

\begin{remark}
 Generative model based phase retrieval has been studied in~\cite{hand2018phase,jagatap2019algorithmic,hyder2019alternating,liu2021towards}. However, for the scenario of SIM with unknown nonlinearity, we follow the settings in~\cite{liu2020generalized}, and as mentioned therein, phase retrieval is beyond the scope of this setup. For the case of a known link function, phase retrievel is also not applicable since its corresponding nonlinear functions are not monotonic. Moreover, it is typically unavoidable for phase retrieval with generative priors to require the strong assumption about the existence of a {\em good initial vector}, whereas for the nonlinear function (whether it is unknown or known) and the corresponding PGD algorithm considered in our work, the initial vector can be {\em arbitrary}.
\end{remark}

\subsection{Notation}

We use upper and lower case boldface letters to denote matrices and vectors respectively. For any positive integer $N$, we write $[N]=\{1,2,\ldots,N\}$ and we use $\bI_N$ to represent an identity matrix in $\bbR^{N\times N}$. A {\em generative model} is a function $G \,:\, \calD\to \bbR^p$, with latent dimension $k$, ambient dimension $p$, and input domain $\calD \subseteq \bbR^k$. We focus on the setting where $k \ll p$. For a set $S \subseteq \bbR^k$ and a generative model $G \,:\,\bbR^k \to \bbR^p$, we write $G(S) = \{ G(\bz) \,:\, \bz \in S  \}$. We use $\|\bX\|_{2 \to 2}$ to denote the spectral norm of a matrix $\bX$. We define the $\ell_q$-ball $B_q^k(r):=\{\bz \in \bbR^k: \|\bz\|_q \le r\}$ for $q \in [0,+\infty]$. The symbols $c, C$ are absolute constants whose values may be different per appearance.

\section{Preliminaries}

We present the definition for a sub-Gaussian random variable.

\begin{definition} \label{def:subg}
 A random variable $X$ is said to be sub-Gaussian if there exists a positive constant $C$ such that $\left(\mathbb{E}\left[|X|^{q}\right]\right)^{1/q} \leq C  \sqrt{q}$ for all $q\geq 1$.  The sub-Gaussian norm of a sub-Gaussian random variable $X$ is defined as $\|X\|_{\psi_2}:=\sup_{q\ge 1} q^{-1/2}\left(\mathbb{E}\left[|X|^{q}\right]\right)^{1/q}$. 
\end{definition}

Throughout this paper, we make the assumption that the generative model $G\,:\, B_2^k(r) \to \bbR^p$ is $L$-Lipschitz continuous, and we fix the structured set $\calK$ to be the range of $G$, i.e., $\calK:= G(B_2^k(r))$. 

In the following, we state the definition of the Two-sided Set-Restricted Eigenvalue Condition (TS-REC), which is adapted from the S-REC proposed in~\cite{bora2017compressed}. 

\begin{definition}
 Let $S \subseteq \bbR^p$. For parameters $\epsilon \in (0,1)$, $\delta \ge 0$, a matrix $\tilde{\bA} \in \bbR^{n \times p}$ is said to satisfy the TS-REC($S,\epsilon,\delta$) if, for every $\bx_1,\bx_2 \in S$, it holds that
 \begin{align}
 (1-\epsilon)\|\bx_1 -\bx_2\|_2 -\delta &\le \left\|\tilde{\bA} (\bx_1-\bx_2)\right\|_2 \le (1+\epsilon)\|\bx_1 -\bx_2\|_2 +\delta,
\end{align}
\end{definition}

Suppose that $\bB\in \bbR^{n\times p}$ has i.i.d.~$\calN(0,1)$ entries. We have the following lemma, which says that $\frac{1}{\sqrt{n}}\bB$ satisfies TS-REC for the set $\calK = G(B_2^k(r))$ with high probability. 
    \begin{lemma}{\em (\hspace{1sp}Adapted from~\cite[Lemma~4.1]{bora2017compressed})}\label{lem:boraTSREC}
     For $\epsilon \in (0,1)$ and $\delta>0$, if $n = \Omega\left(\frac{k}{\epsilon^2} \log \frac{Lr}{\delta}\right)$,\footnote{Here and in subsequent statements of lemmas and theorems, the implied constant is assumed to be sufficiently large.} then a random matrix $\frac{1}{\sqrt{n}}\bB \in \bbR^{n \times p}$ with $b_{ij} \overset{i.i.d.}{\sim} \calN\left(0,1\right)$ satisfies the TS-REC$(\calK,\epsilon,\delta)$ with probability $1-e^{-\Omega(\epsilon^2 n)}$.
\end{lemma}

\section{PGD for Unknown Nonlinearity}\label{sec:unknown}

In this section, we provide theoretical guarantees for a PGD algorithm in the case that the nonlinear link function $f$ is unknown. For this case, we follow~\cite{liu2020generalized} and make the assumptions:   

\begin{itemize}
  \item Let $\bx^* \in \bbR^p$ be the signal to estimate. We assume that $\mu \bx^*$ is contained in the set $\calK = G(B_2^k(r))$, where\footnote{$\mu$ is important for analyzing the recovery performance, but note that the knowledge of $\mu$ cannot be assumed since $f$ is unknown.} 
  \begin{equation}\label{eq:mu_def}
   \mu := \bbE_{g \sim \calN(0,1)}[f(g)g]
  \end{equation}
  is a fixed parameter depending solely on $f$. Since the norm of $\bx^*$ may be absorbed into the unknown $f$, for simplicity of presentation, we assume that $\big\|\bx^*\big\|_2 =1$.

 \item $\ba_i$ are i.i.d.~realizations of a random vector $\ba \sim \calN(\mathbf{0},\bI_p)$, with $\ba$ being independent of $f$. We write the sensing matrix as $\bA = [\ba_1,\ldots,\ba_n]^T \in \bbR^{n \times p}$.
 
 \item We assume the SIM for the (unknown) uncorrupted measurements $y_1,y_2,\ldots,y_n$ as in~\eqref{eq:sim_first}. 

\item Similarly to \cite{plan2017high,liu2020generalized}, the random variable $y := f(\ba^T\bx^*)$ is assumed to be sub-Gaussian with sub-Gaussian norm $\psi$, i.e., 
\begin{equation}\label{eq:psi_def}
 \psi :=\|f(\ba^T\bx^*)\|_{\psi_2} = \|f(g)\|_{\psi_2},
\end{equation}
where $g \sim\calN(0,1)$. Such an assumption will be satisfied, e.g., when $f$ does not grow faster than linearly, i.e., for any $x \in \bbR$, $|f(x)| \le a + b|x|$ for some $a$ and $b$. Hence, various noisy $1$-bit measurement models and non-binary quantization schemes satisfy this assumption~\cite{liu2020generalized}. 

\item In addition to possible random noise in $f$, we allow for adversarial noise that may depend on $\ba$. In particular, instead of observing $\by$ directly, we only assume access to (corrupted) measurements $\tilde{\by} = [\tilde{y}_1,\ldots,\tilde{y}_n]^T \in \bbR^n$ satisfying
 \begin{equation}
    \frac{1}{\sqrt{n}} \|\tilde{\by}-\by\|_2 \le \tau
 \end{equation}
 for some $\tau \ge 0$, where $\by =[y_1,y_2,\ldots,y_n]^T \in \bbR^n$. 

\item To derive an estimate of the signal $\bx^*$ (up to constant scaling), we minimize the linear $\ell_2$ loss over $\calK$:
 \begin{equation}\label{eq:gen_lasso}
  \mathrm{minimize} \quad \calL_1(\bx) := \frac{1}{2n}\|\tilde{\by} -\bA\bx\|_2^2 \quad \text{s.t.} \quad \bx \in \calK.
 \end{equation}
  The above optimization problem is referred to as the generalized Lasso or $\calK$-Lasso. The idea behind using the $\calK$-Lasso to derive an accurate estimate even for nonlinear observations is that the nonlinearity is regarded as noise and the nonlinear observation model can be converted  into a scaled linear model with unconventional noise~\cite{plan2016generalized}.
\end{itemize}
The authors of~\cite{liu2020generalized} provide recovery guarantees with respect to globally optimal solutions of~\eqref{eq:gen_lasso}, but they have not designed practical algorithms to find an optimal solution. Solving~\eqref{eq:gen_lasso} may be practically difficult since in general, $\calK = G(B_2^k(r))$ is not a convex set. In this section, we consider using the following iterative procedure to approximately solve~\eqref{eq:gen_lasso}:
\begin{align}
 \bx^{(t+1)} &= \calP_{\calK}\left(\bx^{(t)} - \nu \cdot \nabla \calL_1\left(\bx^{(t)}\right)\right) \\
 &= \calP_{\calK}\left(\bx^{(t)} - \frac{\nu}{n}\cdot \bA^T \left(\bA \bx^{(t)} - \tilde{\by}\right)\right),\label{eq:pgd_unknown}
\end{align}
where $\calP_\calK(\cdot)$ is the projection function onto $\calK$ and $\nu>0$ is a tuning parameter. For convenience, the corresponding algorithm is described in Algorithm~\ref{algo:pgd_gLasso}. 
\begin{remark}
 We will implicitly assume the exact projection in analysis. Our proof technique does not require $\calP_{\calK}$ to be unique, but only requires it to be a retraction onto the manifold of the generative prior. The exact projection assumption is also made in relevant works including~\cite{hyder2019alternating,liu2022generative,peng2020solving,shah2018solving}. In practice approximate methods might be needed, and both gradient-based projection~\cite{shah2018solving} and GAN-based projection~\cite{raj2019gan} have been shown to be highly effective. Compared to exact projection, in practice global optima of optimization problems like~\eqref{eq:gen_lasso} are typically much more difficult to approximate, and projection-based methods may serve as powerful tools for approximating the global optima. For example, for the simple linear measurement model, it has been numerically verified that performing gradient descent over the latent variable (without using projection) leads to inferior performance and cannot approximate the global optima of~\eqref{eq:gen_lasso} well, whereas a projection-based gradient descent method gives better reconstruction~\cite{shah2018solving}.
\end{remark}

\begin{algorithm}[t]
\caption{A PGD algorithm for approximately solving~\eqref{eq:gen_lasso} (\texttt{PGD-GLasso})}
\label{algo:pgd_gLasso}
{\bf Input}: $\bA$, $\tilde{\by}$, $\nu >0$, number of iterations $T$, generative model $G$, arbitrary initial vector $\bx^{(0)}$ \\
{\bf Procedure}: Iterate as in~\eqref{eq:pgd_unknown} for $t = 0,\ldots,T-1$; return $\bx^{(T)}$
\end{algorithm}

Algorithm~\ref{algo:pgd_gLasso} is identical to the PGD algorithm for solving linear inverse problems using generative models as proposed in~\cite{shah2018solving,peng2020solving}. However, the corresponding analysis is significantly different, since we consider the SIM with an unknown nonlinear function $f$, instead of the simple linear measurement model. In particular, we have the following theorem showing that if $2\mu_1 <1$, Algorithm~\ref{algo:pgd_gLasso} converges linearly and achieves optimal statistical rate, which is roughly of order $\sqrt{(k\log L)/n}$.
The proof of Theorem~\ref{thm:pgd_unknown} is provided in the supplementary material. 
\begin{theorem}\label{thm:pgd_unknown}
 Recall that $\mu$ and $\psi$ are defined in~\eqref{eq:mu_def} and~\eqref{eq:psi_def} respectively. For any $\epsilon \in (0,1)$, letting 
 \begin{equation}
  \mu_1 := \max \{1-\nu(1-\epsilon),\nu (1+\epsilon) -1\}.
 \end{equation}
 For any $\delta > 0$ satisfying $Lr = \Omega(\delta p)$, if $n = \Omega\big(\frac{k}{\epsilon^2}\log \frac{Lr}{\delta}\big)$ and $2 \mu_1 <1$ with $1-2\mu_1 = \Theta(1)$, then for any $t \in \bbN$, with probability $1-e^{-\Omega(\epsilon^2 n)}$, we have 
 \begin{align}\label{eq:pgd_unknown_main}
  \| \bx^{(t)}-\mu \bx^* \|_2 &\le  (2\mu_1)^t \cdot \|\bx^{(0)}-\mu \bx^* \|_2 + C\left(\psi\sqrt{\frac{k\log\frac{Lr}{\delta}}{n}} +\delta + \tau\right).
 \end{align}
\end{theorem}
To ensure that $2\mu_1 < 1$, if $\epsilon$ is chosen to be a sufficiently small positive constant, we should select the parameter $\nu$ from the interval $(0.5,1.5)$, and a good choice of $\nu$ is $\nu =1$. In addition, a $d$-layer neural network generative model typically has Lipschitz constant $L = p^{\Theta(d)}$~\cite{bora2017compressed}, and thus we may set $r = p^{\Theta(d)}$ and $\delta = \frac{1}{p^{\Theta(d)}}$ without affecting the scaling of the term $\log \frac{Lr}{\delta}$ (and the assumption $Lr = \Omega(\delta p)$ is certainly satisfied for fixed $(r,\delta)$).  Hence, if there is no adversarial noise, i.e., $\tau = 0$, and $\psi$ is a fixed constant, we see that after a sufficient number of iterations, Algorithm~\ref{algo:pgd_gLasso} will return a point $\bx^{(T)}$ satisfying $\big\|\bx^{(T)}-\mu\bx^*\big\|_2 = O\big(\sqrt{\frac{k\log\frac{Lr}{\delta}}{n}}\big)$. By the analysis of sample complexity lower bounds for noisy linear CS using generative models~\cite{liu2020information,kamath2020power}, this statistical rate is optimal and cannot be improved without extra assumptions. 

\section{PGD for Known Nonlinearity}\label{sec:known}

In this section, we provide theoretical guarantees for the case when the nonlinear link function $f$ is known. Throughout this section, we make the following assumptions: 

\begin{itemize}
\item Unlike in the case of unknown nonlinearity, we now allow for representation error and assume that the signal $\bx^*$ lies near (but does not need to be exactly contained in) $\calK = G(B_2^k(r))$. Note that since we have precise knowledge of $f$, we do not need to make any assumption on the norm of $\bx^*$.

 \item The sensing matrix $\bA \in \bbR^{n\times p}$ satisfies the following two assumptions with high probability:
 \begin{enumerate}
  \item Johnson-Lindenstrauss embeddings (JLE): For any $\epsilon\in (0,1)$ and any finite set $E \subseteq \bbR^p$ satisfying $n = \Omega\big(\frac{1}{\epsilon^{c_1}}\cdot \log^{c_2}|E|\big)$ for some absolute constants $c_1,c_2$, we have for {\em all} $\bx \in E$  that
  \begin{equation}\label{eq:JLE_assump}
   (1-\epsilon) \|\bx\|_2^2 \le \left\|\frac{1}{\sqrt{n}}\bA\bx\right\|_2^2 \le (1+\epsilon) \|\bx\|_2^2.
  \end{equation}

  \item Bounded spectral norm:  
  For some absolute constant $a$, it holds that
  \begin{equation}\label{eq:BSN_assump}
   \|\bA\|_{2 \to 2} = O(p^a).
  \end{equation} 
 \end{enumerate}
When $\bA$ has independent isotropic sub-Gaussian rows,\footnote{A random vector $\bv $ is said to be isotropic if $\bbE[\bv\bv^T] = \bI_p$.} from Lemma~\cite[Proposition~5.16]{vershynin2010introduction}, we have that when $n = \Omega\big(\frac{1}{\epsilon^2}\log|E|\big)$ (thus $c_1 = 2$ and $c_2 =1$), the event corresponding to~\eqref{eq:JLE_assump} occurs with probability $1-e^{-\Omega(n\epsilon^2)}$. In addition, similarly to~\cite[Corollary~5.35]{vershynin2010introduction}, we have that with probability $1-e^{-\Omega(n)}$, the assumption about bounded spectral norm is satisfied with $a = 0.5$. Moreover, from the theoretical results concerning JLE in~\cite{krahmer2011new} and the standard inequality $\|\bA\|_{2\to 2} \le \max \{\|\bA\|_{1 \to 1}, \|\bA\|_{\infty \to \infty}\}$, we know that when $\bA$ is a subsampled Fourier matrix or a partial Gaussian circulant matrix with random column sign flips and isotropic rows, these two assumptions are also satisfied with high probability for appropriate absolute constants $c_1,c_2$ and $a$. Hence, these assumptions on $\bA$ are significantly more generalized than the i.i.d.~Gaussian assumption made for the case of unknown nonlinearity. Notably, when the two assumptions are satisfied, by a chaining argument~\cite{bora2017compressed,liu2021robust}, the random matrix $\frac{1}{\sqrt{n}}\bA$ satisfies the TS-REC for $\calK$.
 
 \item The (unknown) uncorrupted measurements are generated from the following measurement model:
 \begin{equation}
  y_i = f(\ba_i^T\bx^*) + \eta_i, \quad i = 1,2,\ldots,n, \label{eq:yi_known}
 \end{equation}
where $f \,:\, \bbR \rightarrow \bbR$ is a {\em known (deterministic)} nonlinear function, and $\eta_i$ are additive noise terms. Similarly to~\cite{yang2016sparse,soltani2017fast,zhang2018nonlinear}, we assume that $f$ is monotonic, differentiable, and for all $x \in \bbR$, $f'(x) \in [l,u]$ with $u \ge l >0$ being fixed constants.\footnote{The case that $l\le u <0$ can be similarly handled.}   In addition, we assume that $\eta_i$ are independent realizations of zero-mean sub-Gaussian random variables with maximum sub-Gaussian norm $\sigma$. 

\item We also allow for adversarial noise and assume that for some $\tau \ge 0$, the observed (corrupted) vector $\tilde{\by}$ satisfies
 \begin{equation}
    \frac{1}{\sqrt{n}} \|\tilde{\by}-\by\|_2 \le \tau.
 \end{equation}

\item To estimate the signal $\bx^*$, we utilize the knowledge of $f$, and consider minimizing the nonlinear $\ell_2$ loss over $\calK$:
 \begin{equation}\label{eq:gen_lasso_known}
  \mathrm{minimize} \quad \calL_2(\bx) := \frac{1}{2n}\|\tilde{\by} -f(\bA\bx)\|_2^2 \quad \text{s.t.} \quad \bx \in \calK.
 \end{equation}
\end{itemize}
Under the preceding assumptions, we have the following theorem that gives a recovery guarantee for optimal solutions to~\eqref{eq:gen_lasso_known}. The proof is placed in the supplementary material.
\begin{theorem}\label{thm:opt_known}
 Let $\bar{\bx} = \arg\min_{\bx \in \calK}\|\bx-\bx^*\|_2$. For any $\delta >0$, we have that any solution $\hat{\bx}$ to~\eqref{eq:gen_lasso_known} satisfies
 \begin{equation}\label{eq:ub_opt_known}
  \|\hat{\bx}-\bx^*\|_2 \le O\left(\|\bar{\bx}-\bx^*\|_2 + \sigma\sqrt{\frac{k\log\frac{Lr}{\delta}}{n}} +\tau +\delta\right).
 \end{equation}
\end{theorem}
In~\eqref{eq:ub_opt_known}, the term $\|\bar{\bx}-\bx^*\|_2$ corresponds to the representation error. Similarly to the discussion after Theorem~\ref{thm:pgd_unknown}, we see that when there is no representation error or adversarial noise, and considering $\delta$ being sufficiently small and $\sigma$ being a fixed constant, we obtain the optimal statistical rate, i.e., $\|\hat{\bx}-\bx^*\|_2 = O\big(\sqrt{\frac{k\log\frac{Lr}{\delta}}{n}}\big)$.

Theorem~\ref{thm:opt_known} is concerned with globally optimal solutions of~\eqref{eq:gen_lasso_known}, which are intractable to obtain due to the non-convexity of the corresponding objective function. In the following, we use a PGD algorithm to approximately minimize~\eqref{eq:gen_lasso_known}. In particular, we select an initial vector $\bx^{(0)}$ arbitrarily, and for any non-negative integer $t$, letting
\begin{align}
 \bx^{(t+1)} &= \calP_{\calK}\left(\bx^{(t)} - \zeta\cdot \nabla \calL_2\left(\bx^{(t)}\right)\right) \\
 &= \calP_{\calK}\left(\bx^{(t)} - \frac{\zeta}{n}\cdot \bA^T \left((f(\bA\bx^{(t)})-\tilde{\by}) \odot f'(\bA\bx^{(t)})\right)\right),\label{eq:pgd_known}
\end{align}
where $\zeta >0$ is the step size and ``$\odot$'' represents element-wise product. The corresponding algorithm is described in Algorithm~\ref{algo:pgd_known} for convenience.

\begin{algorithm}[t]
\caption{A PGD algorithm for approximately solving~\eqref{eq:gen_lasso_known} (\texttt{PGD-NLasso})}
\label{algo:pgd_known}
{\bf Input}: $\bA$, $\tilde{\by}$, $\zeta >0$, number of iterations $T$, generative model $G$, arbitrary initial vector $\bx^{(0)}$ \\
{\bf Procedure}: Iterate as~\eqref{eq:pgd_known} for $t = 0,\ldots,T-1$; return $\bx^{(T)}$
\end{algorithm}

Next, we present the following theorem, which establishes a theoretical guarantee similar to that of Theorem~\ref{thm:pgd_unknown}, except that there is an extra $\|\bar{\bx}-\bx^*\|_2$ term  corresponding to representation error in the upper bound. The proof of Theorem~\ref{thm:pgd} can be found in the supplementary material.
\begin{theorem}\label{thm:pgd}
Let $\bar{\bx}=\arg\min_{\bx \in \calK} \|\bx-\bx^*\|_2$. For any $\delta>0$ and $\epsilon \in (0,1)$ that is sufficiently small, letting
 \begin{equation}\label{eq:def_C_eps}
  \mu_2 := \max \left\{1-\zeta l^2(1-\epsilon), \zeta u^2 (1+\epsilon) -1\right\}.
 \end{equation}
 If $2\mu_2 < 1$ with $1-2\mu_2 =\Theta(1)$, we have for all $t \in \bbN$ that  
\begin{align}
 &\|\bx^{(t)}-\bx^*\|_2 \le  \left(2\mu_2\right)^t \cdot \|\bx^{(0)}-\bx^*\|_2 + O\left(\|\bar{\bx}-\bx^*\|_2 + \sigma\sqrt{\frac{k\log\frac{Lr}{\delta}}{n}}+\tau+\delta\right).
\end{align}
\end{theorem}

\section{Experiments} 

\begin{figure}[t]
\begin{center}
\includegraphics[width=1.0\columnwidth]{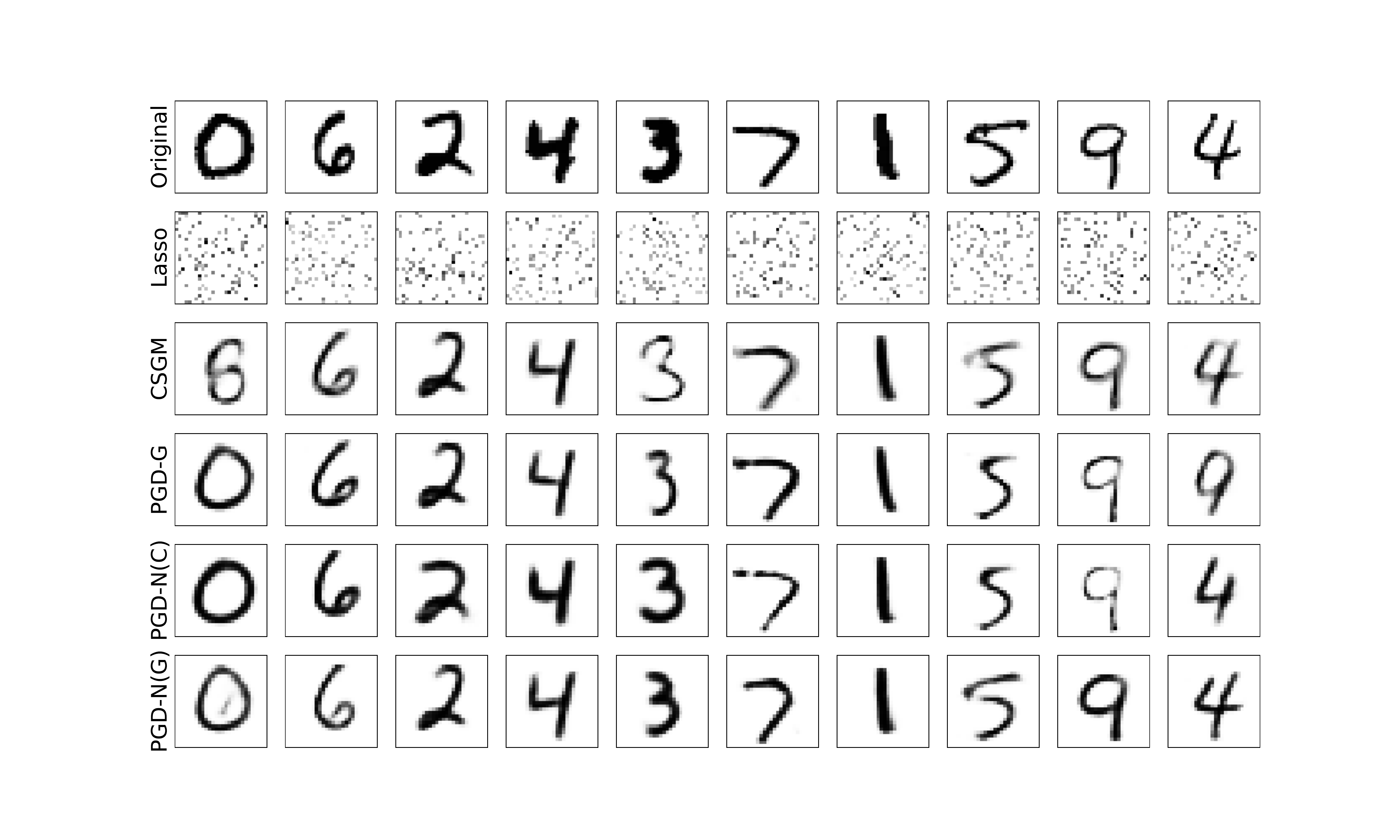}
\caption{Examples of reconstructed images of the MNIST dataset with $n = 100$ measurements and $p = 784$ dimensional vectors.} \label{fig:mnist_imgs_sign}
\end{center}
\end{figure}

In this section, we empirically evaluate the performance of Algorithm~\ref{algo:pgd_gLasso} (\texttt{PGD-GLasso}; abbreviated to~\texttt{PGD-G}) for the linear least-squares estimator~\eqref{eq:gen_lasso} and Algorithm~\ref{algo:pgd_known} (\texttt{PGD-NLasso}; abbreviated to~\texttt{PGD-N}) for the nonlinear least-squares estimator~\eqref{eq:gen_lasso_known} on the MNIST~\cite{lecun1998gradient} and CelebA~\cite{liu2015deep} datasets. For both datasets, we use the generative models pre-trained by the authors of~\cite{bora2017compressed}. For~\texttt{PGD-G}, the step size $\nu$ is set to be $1$. For~\texttt{PGD-N}, we set the step size $\zeta$ to be $0.2$. For both algorithms, the total number of iterations $T$ is set to be $30$. On the MNIST dataset, we do $5$ random restarts, and pick the best estimate among these random restarts. The generative model $G$ is set to be a variational autoencoder (VAE) model with a latent dimension of $k = 20$. The projection step $\calP_{\calK}(\cdot)$ with $\calK$ being the range of $G$ is approximated using gradient descent, performed using the Adam optimizer with a learning rate of $0.03$ and $200$ steps. On the CelebA dataset, we use a Deep Convolutional Generative Adversarial Networks (DCGAN) generative model with a latent dimension of $k = 100$. We select the best estimate among $2$ random restarts. An Adam optimizer with $100$ steps and a learning rate of $0.1$ is used for the projection operator $\calP_{\calK}(\cdot)$. Throughout this section, for simplicity, we consider the case that there is no adversarial noise, i.e., $\tau =0$. Since for unknown nonlinearity, the signal is recovered up to a scalar ambiguity, to compare performance across algorithms, we use a scale-invariant metric named Cosine Similarity defined as $\mathrm{Cos} \big(\bx^*,\bx^{(T)}\big) := \frac{\big\langle\bx^*,\bx^{(T)} \big\rangle}{\|\bx^*\|_2\|\bx^{(T)}\|_2}$, where $\bx^*$ is the signal vector to estimate, and $\bx^{(T)}$ denotes the output vector of an algorithm. We use Python 2.7 and TensorFlow 1.0.1, with a NVIDIA Tesla K80 24GB GPU.

\begin{figure}[t]
\begin{center}
\includegraphics[width=1.0\columnwidth]{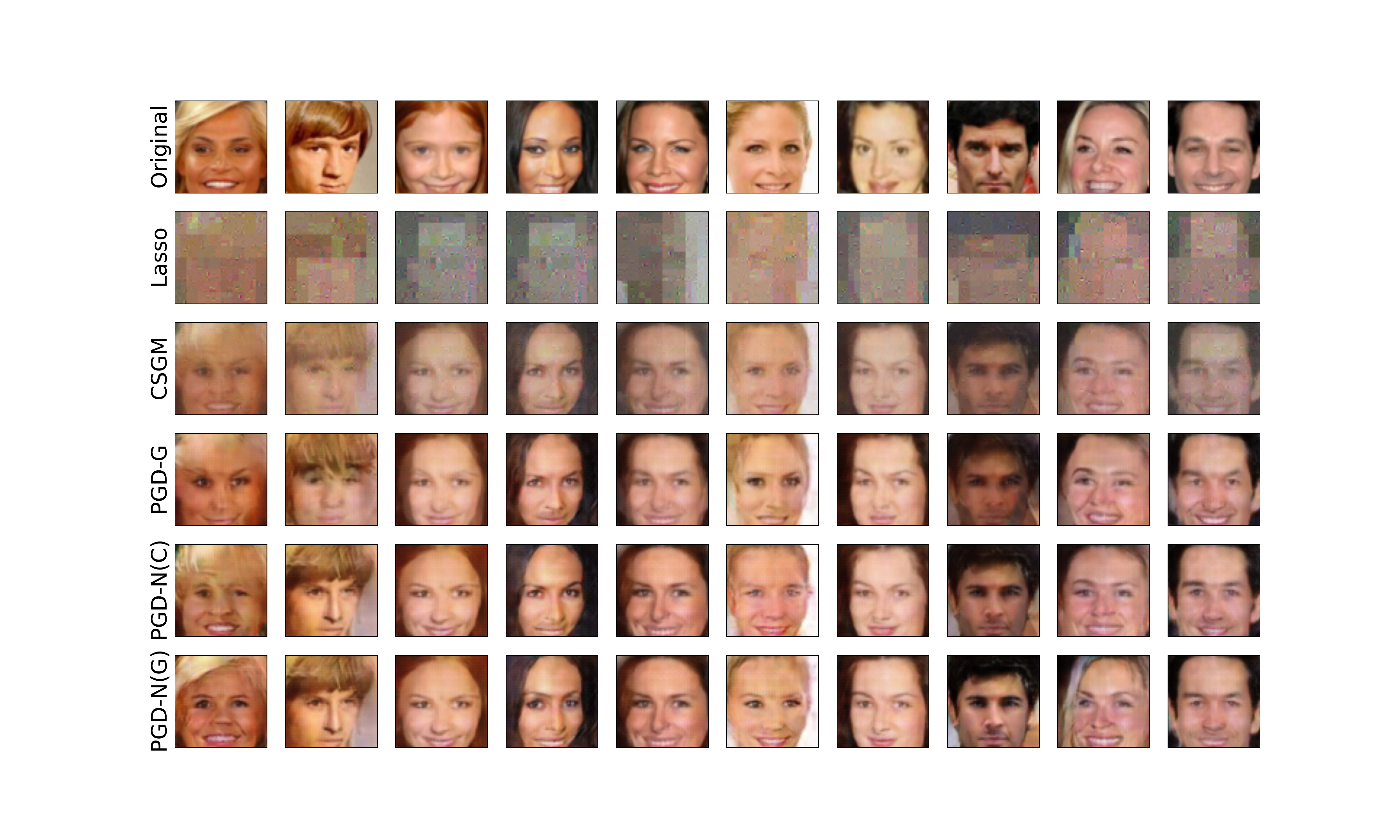}
\caption{Reconstructed images of the CelebA dataset with $n = 500$ measurements and $p = 12288$ dimensional vectors.}\label{fig:celebA_imgs} 
\end{center}
\end{figure}

We follow the measurement model~\eqref{eq:yi_known} to generate the observations, with $f(x) = 2x + 0.5 \cos(x)$ and the noises being independent zero-mean Gaussian with standard deviation $\sigma$. We observe that $f$ is smooth and monotonically increasing with $l=1.5$ and $u=2.5$. Thus, we perform~\texttt{PGD-N} on the generated data. In addition, since the assumption made for unknown nonlinearity that $f(g)$ for $g \sim \calN(0,1)$ is sub-Gaussian is satisfied, we also compare with~\texttt{PGD-G}. The standard deviation $\sigma$ is set to be $0.1$ for the MNIST dataset, and $0.01$ for the CelebA dataset. The baseline is the~\texttt{Lasso} using $2$D Discrete Cosine Transform ($2$D-DCT) basis~\cite{Tib96} and the method for linear inverse problem with generative models proposed in~\cite{bora2017compressed} (denoted by~\texttt{CSGM}). For~\texttt{Lasso},~\texttt{CSGM} and~\texttt{PGD-G}, the sensing matrix $\bA$ is assumed to contain~i.i.d. standard Gaussian entries. Moreover, since for a known and monotonic nonlinear link function, we allow for a wide variety of distributions on the sensing vectors, for~\texttt{PGD-N}, we consider both cases where $\bA$ is a standard Gaussian matrix or a partial Gaussian circulant matrix similar to that in~\cite{liu2021towards}. The corresponding Algorithm~\ref{algo:pgd_known} is denoted by~\texttt{PGD-N(G)} or \texttt{PGD-N(C)} respectively, where ``G'' refers to the standard Gaussian matrix and ``C'' refers to the partial Gaussian circulant matrix.

We perform experiments to compare the performance of these algorithms, and the reconstructed results are reported in Figures~\ref{fig:mnist_imgs_sign},~\ref{fig:celebA_imgs}, and~\ref{fig:quant}. We observe that for our settings, the sparsity-based method~\texttt{Lasso} always attains poor reconstructions, while all three generative model based PGD methods attain accurate reconstructions even when the number of measurements $n$ is small compared to the ambient dimension $p$. In addition, from Figure~\ref{fig:quant}(a), we observe that on the MINST dataset,~\texttt{PGD-N(G)} and~\texttt{PGD-N(C)} lead to similar Cosine Similarity, and they are clearly better than~\texttt{CSGM} and~\texttt{PGD-G} when $n < 300$. This is not surprising since both~\texttt{CSGM} and~\texttt{PGD-G} does not make use of the knowledge of the nonlinear link function. From Figures~\ref{fig:celebA_imgs} and~\ref{fig:quant}(b), we see that for the CelebA dataset, three generative prior based PGD methods give similar reconstructed images, with the Cosine Similarity corresponding to~\texttt{PGD-N(G)} and~\texttt{PGD-N(C)} being slightly higher than that of~\texttt{PGD-G}, and all of them lead to clearly higher Cosine Similarity compared to that of~\texttt{CSGM}.

\begin{figure}[t]
\begin{center}
\begin{tabular}{cc}
\includegraphics[height=0.18\textwidth]{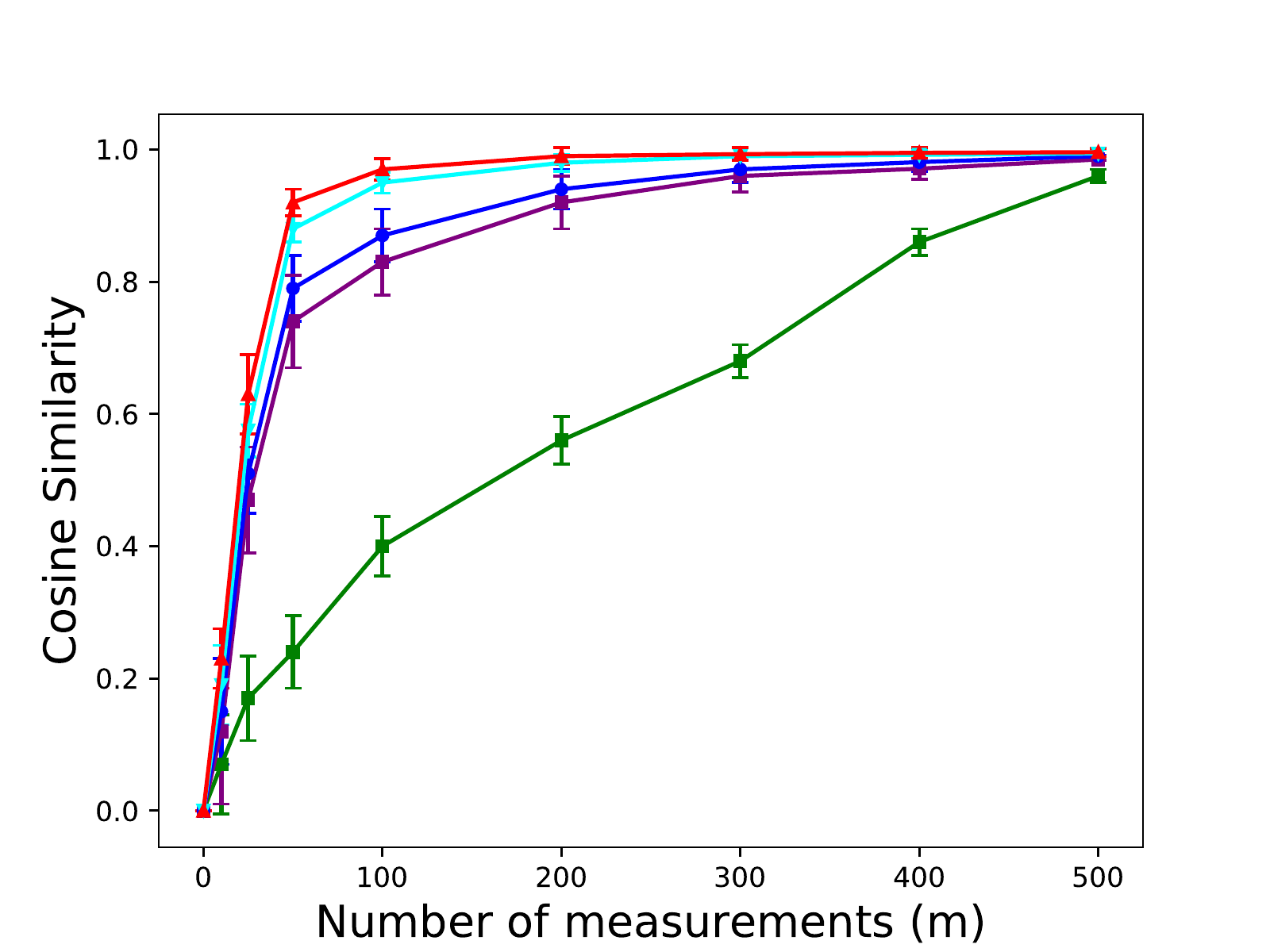} & \hspace{-0.5cm}
\includegraphics[height=0.18\textwidth]{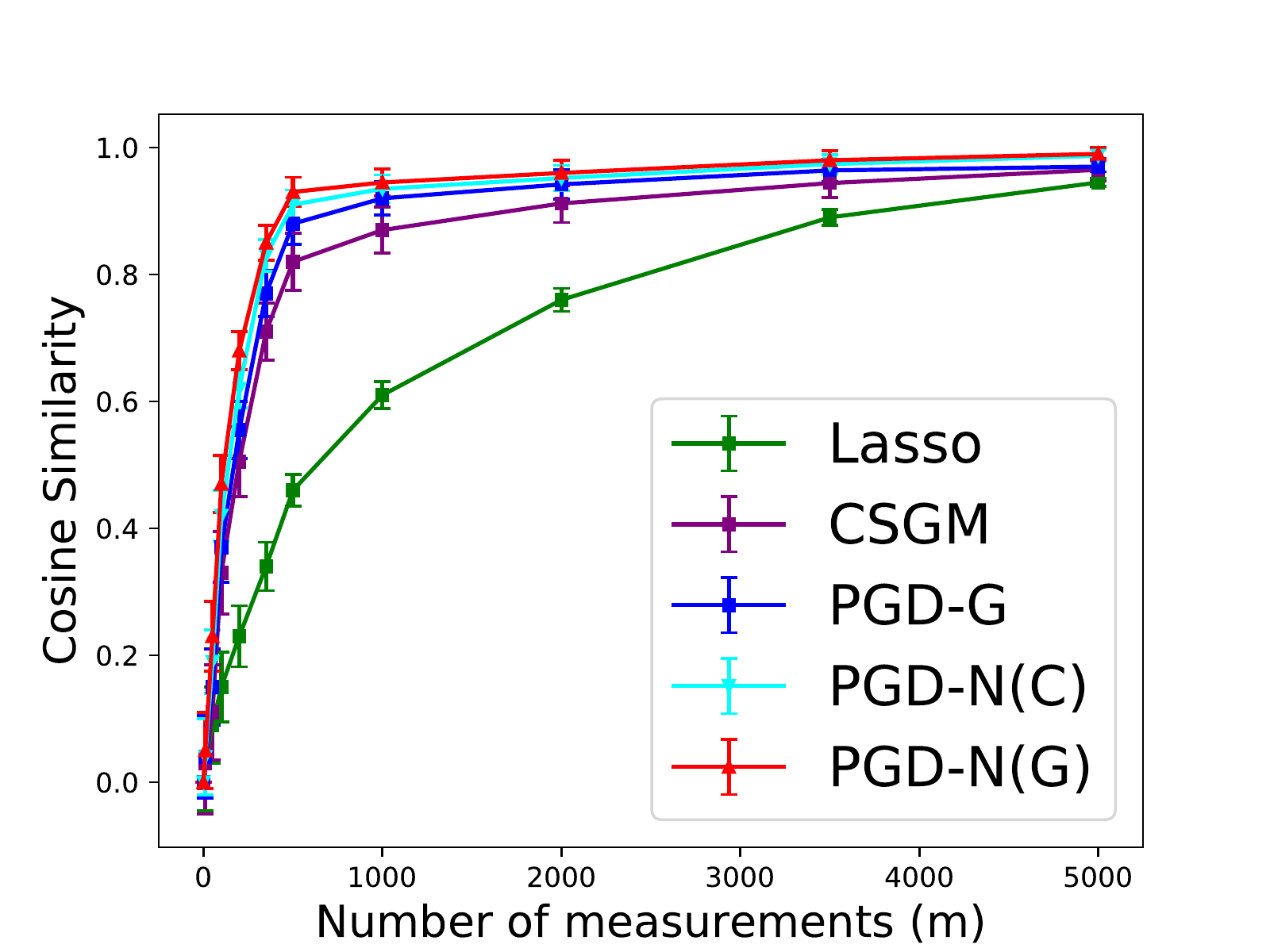} \\
{\small (a) Results on MNIST} & {\small (b) Results on CelebA}
\end{tabular}
\caption{Quantitative comparisons for the performance of \texttt{PGD-G} and \texttt{PGD-N} according to Cosine Similarity.} \label{fig:quant}
\end{center}
\end{figure} 

Numerical results for noisy $1$-bit measurements are presented in the supplementary material.

\section{Conclusion}

We have proposed PGD algorithms to solve generative model based nonlinear inverse problems, and we have provided theoretical guarantees for these algorithms for both unknown and known link functions. In particular, these algorithms are guaranteed to converge linearly to points achieving optimal statistical rate in spite of the model nonlinearity. 

\section*{Acknowledgment} We sincerely thank the three anonymous reviewers for their careful reading and insightful comments. We are extremely grateful to Dr. Jonathan Scarlett for proofreading the manuscript and giving valuable suggestions.

\newpage

\appendices

\section{Proof of Theorem~\ref{thm:pgd_unknown} (PGD for Unknown Nonlinearity)}

Before proving the theorem, we present some auxiliary lemmas.

\subsection{Auxiliary Results for Theorem~\ref{thm:pgd_unknown}}

First, we present a standard definition for sub-exponential random variables.

\begin{definition}
 A random variable $X$ is said to be sub-exponential if there exists a positive constant $C$ such that $\left(\bbE\left[|X|^q\right]\right)^{\frac{1}{q}} \le C q$ for all $q \ge 1$. The sub-exponential norm of $X$ is defined as $\|X\|_{\psi_1} := \sup_{q \ge 1} p^{-1} \left(\bbE\left[|X|^q\right]\right)^{\frac{1}{q}}$.
\end{definition}

The following lemma states that the product of two sub-Gaussian random variables is sub-exponential.
\begin{lemma}{\em (\hspace{1sp}\cite{vershynin2010introduction})}\label{lem:prod_subGs}
  Let $X$ and $Y$ be sub-Gaussian random variables (not necessarily independent). Then $XY$ is sub-exponential, and satisfies
  \begin{equation}
   \|XY\|_{\psi_1} \le \|X\|_{\psi_2}\|Y\|_{\psi_2}.
  \end{equation}
\end{lemma}

The following lemma provides a useful concentration inequality for the sum of independent zero-mean sub-exponential random variables.
\begin{lemma}{\em (\hspace{1sp}\cite[Proposition~5.16]{vershynin2010introduction})}\label{lem:large_dev}
Let $X_{1}, \ldots , X_{N}$ be independent zero-mean sub-exponential random variables, and $K = \max_{i} \|X_{i} \|_{\psi_{1}}$. Then for every $\balpha = [\alpha_1,\ldots,\alpha_N]^T \in \bbR^N$ and $\epsilon \geq 0$, it holds that
\begin{align}
 & \mathbb{P}\bigg( \Big|\sum_{i=1}^{N}\alpha_i X_{i}\Big|\ge \epsilon\bigg)  \leq 2  \exp \left(-c \cdot \mathrm{min}\Big(\frac{\epsilon^{2}}{K^{2}\|\balpha\|_2^2},\frac{\epsilon}{K\|\balpha\|_\infty}\Big)\right). \label{eq:subexp}
\end{align}
\end{lemma}

First, based on Lemmas~\ref{lem:prod_subGs} and~\ref{lem:large_dev}, we derive the following lemma.
\begin{lemma}\label{lem:aux_main_unknown}
 For any $\bx_1,\bx_2 \in \bbR^p$ and any $\epsilon \in (0,1)$, we have with probability $1-e^{-\Omega(n\epsilon^2)}$ that
 \begin{equation}\label{eq:aux_main_unknown}
  \left|\left\langle \left(\bI_p - \frac{\nu}{n}\bA^T\bA\right)\bx_1,\bx_2\right\rangle\right| \le \mu_1 \|\bx_1\|_2\|\bx_2\|_2,
 \end{equation}
where $\mu_1 := \max \{1-\nu(1-\epsilon),\nu (1+\epsilon) -1\}$.
\end{lemma}
\begin{proof}
 We have 
 \begin{equation}
 \left\langle\frac{1}{n}\bA^T\bA\bx_1,\bx_2\right\rangle = \frac{1}{n}\sum_{i=1}^n (\ba_i^T\bx_1)(\ba_i^T\bx_2). 
 \end{equation}
From Lemma~\ref{lem:prod_subGs}, for any $i \in [n]$, $(\ba_i^T\bx_1)(\ba_i^T\bx_2)$ is sub-exponential with mean $\bx_1^T \bx_2$ and the sub-exponential norm being upper bounded by $C\|\bx_1\|_2\|\bx_2\|_2$. From Lemma~\ref{lem:large_dev}, we have with probability $1-e^{-\Omega(n\epsilon^2)}$ that
\begin{align}
 &\left|\left\langle\frac{1}{n}\bA^T\bA\bx_1,\bx_2\right\rangle - \bx_1^T\bx_2\right|  = \left|\frac{1}{n}\sum_{i=1}^n \left((\ba_i^T\bx_1)(\ba_i^T\bx_2) - \bx_1^T\bx_2\right)\right| \le \epsilon\|\bx_1\|_2\|\bx_2\|_2. 
\end{align}
This implies 
\begin{align}
 \nu \bx_1^T\bx_2 - \nu \epsilon\|\bx_1\|_2\|\bx_2\|_2 & \le \left\langle\frac{\nu}{n}\bA^T\bA\bx_1,\bx_2\right\rangle \le \nu \bx_1^T\bx_2 + \nu \epsilon\|\bx_1\|_2\|\bx_2\|_2,
\end{align}
and thus
\begin{align}
 & (1-\nu) \bx_1^T\bx_2 - \nu \epsilon\|\bx_1\|_2\|\bx_2\|_2 \le \left\langle \left(\bI_p - \frac{\nu}{n}\bA^T\bA\right)\bx_1,\bx_2\right\rangle  \le (1-\nu) \bx_1^T\bx_2 + \nu \epsilon\|\bx_1\|_2\|\bx_2\|_2. 
\end{align}
Using the inequality $|\bx_1^T\bx_2| \le \|\bx_1\|_2\|\bx_2\|_2$, we obtain~\eqref{eq:aux_main_unknown}. 
\end{proof}

Next, we present the following useful lemma.
\begin{lemma}{\em \hspace{1sp}\cite[Lemma~3]{liu2020generalized}}\label{lem:bxStarbxHat}
 Fix any $\bs \in \bbR^p$ satisfying $\|\bs\|_2=1$ and let $\bar{\by} := f(\bA \bs)$. Suppose that some $\tilde{\bx} \in \calK$ is selected depending on $\bar{\by}$ and $\bA$. For any $\delta >0$, if $Lr = \Omega(\delta p)$ and $n = \Omega\left(k \log \frac{Lr}{\delta}\right)$, then with probability $1-e^{-\Omega\left(k \log \frac{Lr}{\delta}\right)}$, it holds that
 \begin{align}
  &\left\langle \frac{1}{n}\bA^T(\bar{\by} - \mu\bA \bs), \tilde{\bx}-\mu \bs \right \rangle \le O\left(\psi \sqrt{\frac{k\log \frac{Lr}{\delta}}{n}}\right)\|\tilde{\bx} - \mu\bs\|_2 + O\left(\delta\psi\sqrt{\frac{k\log \frac{Lr}{\delta}}{n}}\right).
 \end{align}
\end{lemma}

With the above lemmas in place, we are now ready to prove Theorem~\ref{thm:pgd_unknown}.

\subsection{Proof of Theorem~\ref{thm:pgd_unknown}}
\label{app:proofThm1}

For $t \ge 0$, let 
 \begin{equation}
  \tilde{\bx}^{(t+1)} = \bx^{(t)} - \frac{\nu}{n}  \bA^T \left(\bA \bx^{(t)} - \tilde{\by}\right). 
 \end{equation}
Because $\bx^{(t+1)} = \calP_{\calK}\big(\tilde{\bx}^{(t+1)}\big)$ and $\mu\bx^* \in \calK$, we have
\begin{equation}
 \|\tilde{\bx}^{(t+1)} - \bx^{(t+1)}\|_2^2 \le \|\tilde{\bx}^{(t+1)} - \mu\bx^*\|_2^2.
\end{equation}
Equivalently, 
\begin{equation}
 \left\|\left(\tilde{\bx}^{(t+1)} - \mu\bx^*\right) + \left(\mu \bx^* - \bx^{(t+1)}\right) \right\|_2^2 \le \left\|\tilde{\bx}^{(t+1)} - \mu\bx^*\right\|_2^2,
\end{equation}
which gives
\begin{equation}
 \left\|\mu \bx^* - \bx^{(t+1)} \right\|_2^2 + 2\left\langle\tilde{\bx}^{(t+1)} - \mu\bx^*, \mu \bx^* - \bx^{(t+1)}\right\rangle \le 0.
\end{equation}
Therefore,
\begin{align}
  &\|\bx^{(t+1)} - \mu\bx^*\|_2^2 \le 2 \left\langle\tilde{\bx}^{(t+1)} - \mu \bx^*,  \bx^{(t+1)} - \mu\bx^*\right\rangle \\
  & = 2 \left\langle \bx^{(t)} - \mu \bx^* - \frac{\nu}{n} \bA^T \left(\bA \bx^{(t)} - \tilde{\by}\right),  \bx^{(t+1)} - \mu\bx^*\right\rangle \\
  & = 2 \left\langle \bx^{(t)} - \mu \bx^* - \frac{\nu}{n}  \bA^T \left(\bA \bx^{(t)} - \mu \bA \bx^*\right),  \bx^{(t+1)} - \mu\bx^*\right\rangle - 2 \left\langle \frac{\nu}{n}  \bA^T \left(\mu \bA \bx^* - \by \right),  \bx^{(t+1)} - \mu\bx^*\right\rangle \nonumber \\
  & \indent - 2 \left\langle \frac{\nu}{n}  \bA^T \left(\by - \tilde{\by} \right),  \bx^{(t+1)} - \mu\bx^*\right\rangle \\
  & = 2 \left\langle \left(\bI_p - \frac{\nu}{n} \bA^T\bA\right)( \bx^{(t)} - \mu \bx^*),  \bx^{(t+1)} - \mu\bx^*\right\rangle + 2\nu \left\langle \frac{1}{n} \bA^T \left(\by -\mu \bA \bx^* \right),  \bx^{(t+1)} - \mu\bx^*\right\rangle \nonumber \\
  & \indent + 2 \nu \left\langle \frac{1}{n}  \bA^T \left(\tilde{\by} - \by \right),  \bx^{(t+1)} - \mu\bx^*\right\rangle. \label{eq:main_ineq_unknown}
\end{align}
From Lemma~\ref{lem:bxStarbxHat}, we obtain that for any $\delta >0$, if $Lr = \Omega(\delta p)$ and $n =\Omega\big(k \log \frac{Lr}{\delta}\big)$, then with probability $1-e^{-\Omega\big(k \log \frac{Lr}{\delta}\big)}$, it holds that 
\begin{equation}\label{eq:genLasso_ineq_unknown}
\left| \left\langle \frac{1}{n} \bA^T \left(\by -\mu \bA \bx^* \right),  \bx^{(t+1)} - \mu\bx^*\right\rangle\right| \le O\left(\psi \sqrt{\frac{k \log\frac{Lr}{\delta}}{n}}\right) \|\bx^{(t+1)} - \mu\bx^*\|_2 + O\left(\delta\psi \sqrt{\frac{k \log\frac{Lr}{\delta}}{n}}\right).
\end{equation}
In addition, from the Cauchy-Schwarz inequality, we obtain with probability that $1-e^{-\Omega(n)}$ that 
\begin{align}
 \left|\left\langle \frac{1}{n}  \bA^T \left(\tilde{\by} - \by \right),  \bx^{(t+1)} - \mu\bx^*\right\rangle\right| &\le \left\|\frac{1}{\sqrt{n}}(\tilde{\by} - \by)\right\|_2 \cdot \left\|\frac{1}{\sqrt{n}}\bA(\bx^{(t+1)} - \mu\bx^*)\right\|_2 \\
 & \le \tau \cdot O(\|\bx^{(t+1)} - \mu\bx^*\|_2 +\delta),\label{eq:adverNoise_ineq_unknown}
\end{align}
where~\eqref{eq:adverNoise_ineq_unknown} is from setting $\alpha = 0.5$ in Lemma~\ref{lem:boraTSREC}. From~\eqref{eq:main_ineq_unknown}, we observe that it remains to derive an upper bound for $\big|\big\langle \big(\bI_p - \frac{\nu}{n} \bA^T\bA\big)( \bx^{(t)} - \mu \bx^*),  \bx^{(t+1)} - \mu\bx^*\big\rangle\big|$. To achieve this goal, we consider using a chain of nets. For any positive integer $q$, let $M = M_0 \subseteq M_1 \subseteq \ldots \subseteq M_q$ be a chain of nets of $B_2^k(r)$ such that $M_i$ is a $\frac{\delta_i}{L}$-net with $\delta_i = \frac{\delta}{2^i}$. There exists such a chain of nets with \cite[Lemma~5.2]{vershynin2010introduction}
    \begin{equation}
        \log |M_i| \le k \log\frac{4Lr}{\delta_i}. \label{eq:net_size_main}
    \end{equation}
    Then, by the $L$-Lipschitz assumption on $G$, we have for any $i \in [q]$ that $G(M_i)$ is a $\delta_i$-net of $G(B_2^k(r))$. For $\bx^{(t)},\bx^{(t+1)} \in G(B_2^k(r))$, we write 
    \begin{align}
     \bx^{(t)} &= (\bx^{(t)} - \bu_q) + \sum_{i=1}^q (\bu_i - \bu_{i-1}) + \bu_0,\\
     \bx^{(t+1)}& = (\bx^{(t+1)} - \bv_q) + \sum_{i=1}^q (\bv_i - \bv_{i-1}) + \bv_0,
    \end{align}
    where for $i \in [q]$, $\bu_i,\bv_i \in  G(M_i)$, $\|\bx^{(t)}- \bu_q\|_2 \le \frac{\delta}{2^q}$, $\|\bx^{(t+1)}- \bv_q\|_2 \le \frac{\delta}{2^q}$, $\|\bu_i - \bu_{i-1}\|_2 \le \frac{\delta}{2^{i-1}}$, $\|\bv_i - \bv_{i-1}\|_2 \le \frac{\delta}{2^{i-1}}$. The triangle inequality yields 
    \begin{equation}\label{eq:ttp1_u0v0}
     \|\bx^{(t)}- \bu_0\|_2 \le 2\delta, \quad \|\bx^{(t+1)}- \bv_0\|_2 \le 2\delta.
    \end{equation}
    Let $\bW_\nu = \bI_p - \frac{\nu}{n} \bA^T\bA$. Then, we have 
    \begin{align}
     & \left\langle \left(\bI_p - \frac{\nu}{n} \bA^T\bA\right)( \bx^{(t)} - \mu \bx^*),  \bx^{(t+1)} - \mu\bx^*\right\rangle = \left\langle \bW_\nu ( \bx^{(t)} - \mu \bx^*),  \bx^{(t+1)} - \mu\bx^*\right\rangle \\
     & = \left\langle \bW_\nu ( \bx^{(t)} - \bu_q),  \bx^{(t+1)} - \mu\bx^*\right\rangle  + \left\langle \bW_\nu ( \bu_0 - \mu \bx^*),  \bx^{(t+1)} - \mu\bx^*\right\rangle + \sum_{i=1}^q\left\langle \bW_\nu ( \bu_{i} - \bu_{i-1}),  \bx^{(t+1)} - \mu\bx^*\right\rangle.\label{eq:ineq_threeTerms_unknown}
    \end{align}
In the following, we control the three terms in~\eqref{eq:ineq_threeTerms_unknown} separately.
\begin{itemize}
 \item The first term: We have with probability at least $1-e^{-n/2}$ that $\|\bA\|_{2\to 2} \le 2\sqrt{n} + \sqrt{p}$~\cite[Corollary~5.35]{vershynin2010introduction}, which implies $\|\bW_\nu\|_{2\to 2} = \big\|\bI_p - \frac{\nu}{n} \bA^T\bA\big\|_{2\to 2} = O\big(\frac{p}{n}\big)$. Therefore,
 \begin{align}
  \left|\left\langle \bW_\nu ( \bx^{(t)} - \bu_q),  \bx^{(t+1)} - \mu\bx^*\right\rangle\right| & \le \|\bW_\nu\|_{2\to 2} \cdot \|\bx^{(t)} - \bu_q\|_2 \cdot \|\bx^{(t+1)} - \mu\bx^*\|_2 \\
  & \le O\big(\frac{p}{n}\big) \cdot \frac{\delta}{2^q}\cdot \|\bx^{(t+1)} - \mu\bx^*\|_2. 
 \end{align}
Choosing $q = \lceil \log_2 p \rceil$, we obtain
\begin{equation}\label{eq:threeTerms_unknown_first}
 \left|\left\langle \bW_\nu ( \bx^{(t)} - \bu_q),  \bx^{(t+1)} - \mu\bx^*\right\rangle\right| \le C \delta \|\bx^{(t+1)} - \mu\bx^*\|_2. 
\end{equation}

\item The second term: We have 
\begin{align}
 \left\langle \bW_\nu ( \bu_0 - \mu \bx^*),  \bx^{(t+1)} - \mu\bx^*\right\rangle & = \left\langle \bW_\nu ( \bu_0 - \mu \bx^*),  \bx^{(t+1)} - \bv_q\right\rangle + \left\langle \bW_\nu ( \bu_0 - \mu \bx^*),  \bv_0 - \mu\bx^*\right\rangle  \nonumber \\
 & \indent + \sum_{i=1}^q \left\langle \bW_\nu ( \bu_0 - \mu \bx^*),  \bv_i - \bv_{i-1}\right\rangle.
\end{align}
Similarly to~\eqref{eq:threeTerms_unknown_first}, we have that when $q = \lceil \log_2 p \rceil$, with probability at least $1-e^{-n/2}$,  
\begin{equation}\label{eq:threeTerms_unknown_second_first}
 \left|\left\langle \bW_\nu ( \bu_0 - \mu \bx^*),  \bx^{(t+1)} - \bv_q\right\rangle\right| \le C \delta \|\bu_0 - \mu \bx^*\|_2 \le C \delta (\|\bx^{(t)} - \mu \bx^*\|_2  + 2\delta).
\end{equation}
From Lemma~\ref{lem:aux_main_unknown}, and taking a union bound over $G(M) \times G(M)$, we have that if $n = \Omega\big(\frac{k}{\epsilon^2}\log \frac{Lr}{\delta}\big)$, then with probability $1-e^{-\Omega(n\epsilon^2)}$, 
\begin{align}
 \left\langle \bW_\nu ( \bu_0 - \mu \bx^*),  \bv_0 - \mu\bx^*\right\rangle & \le \mu_1 \|\bu_0 - \mu \bx^*\|_2 \|\bv_0 - \mu\bx^*\|_2 \\
 & \le\mu_1 (\|\bx^{(t)} - \mu \bx^*\|_2+ 2\delta) (\|\bx^{(t+1)} - \mu\bx^*\|_2 +2\delta),\label{eq:threeTerms_unknown_second_second}
\end{align}
where $\mu_1 = \max \{1-\nu(1-\epsilon),\nu (1+\epsilon) -1\}$. 
Moreover, using Lemma~\ref{lem:aux_main_unknown} and the results in the chaining argument in~\cite{bora2017compressed,liu2020sample}, we have that when $n = \Omega\big(\frac{k}{\epsilon^2}\log \frac{Lr}{\delta}\big)$, with probability $1-e^{-\Omega(n\epsilon^2)}$, 
\begin{equation}\label{eq:threeTerms_unknown_second_third}
 \left|\sum_{i=1}^q \left\langle \bW_\nu ( \bu_0 - \mu \bx^*),  \bv_i - \bv_{i-1}\right\rangle\right| \le C \delta \|\bu_0 - \mu \bx^*\|_2 \le C \delta (\|\bx^{(t)} - \mu \bx^*\|_2  + 2\delta).
\end{equation}

\item The third term: Again, using Lemma~\ref{lem:aux_main_unknown} and the results in the chaining argument in~\cite{bora2017compressed,liu2020sample}, we have that when $n = \Omega\big(\frac{k}{\epsilon^2}\log \frac{Lr}{\delta}\big)$, with probability $1-e^{-\Omega(n\epsilon^2)}$, 
\begin{equation}\label{eq:threeTerms_unknown_third}
 \left|\sum_{i=1}^q\left\langle \bW_\nu ( \bu_{i} - \bu_{i-1}),  \bx^{(t+1)} - \mu\bx^*\right\rangle\right| \le C \delta (\|\bx^{(t+1)} - \mu\bx^*\|_2 +2\delta). 
\end{equation}
\end{itemize}
Combining~\eqref{eq:main_ineq_unknown},~\eqref{eq:genLasso_ineq_unknown},~\eqref{eq:adverNoise_ineq_unknown},~\eqref{eq:ineq_threeTerms_unknown},~\eqref{eq:threeTerms_unknown_first},~\eqref{eq:threeTerms_unknown_second_first},~\eqref{eq:threeTerms_unknown_second_second},~\eqref{eq:threeTerms_unknown_second_third}, and~\eqref{eq:threeTerms_unknown_third}, we obtain that when $Lr= \Omega(\delta p)$ and $n = \Omega\big(\frac{k}{\epsilon^2}\log \frac{Lr}{\delta}\big)$, with probability $1-e^{-\Omega(n\epsilon^2)}$, 
\begin{equation}\label{eq:combining_unknown}
 \|\bx^{(t+1)}- \mu \bx^*\|_2^2 \le \left(2\mu_1\|\bx^{(t)}- \mu \bx^*\|_2 + O\left(\psi\sqrt{\frac{k\log\frac{Lr}{\delta}}{n}} +\delta + \tau\right)\right) (\|\bx^{(t+1)}- \mu \bx^*\|_2 + C\delta),
\end{equation}
which gives
\begin{equation}\label{eq:origin_54}
 \|\bx^{(t+1)}- \mu \bx^*\|_2 \le 2\mu_1\|\bx^{(t)}- \mu \bx^*\|_2 + O\left(\psi\sqrt{\frac{k\log\frac{Lr}{\delta}}{n}} +\delta + \tau\right),
\end{equation}
and leads to our desired result.\footnote{~\eqref{eq:origin_54} is of the form $z_{t} \le \alpha_1 z_{t-1} + \alpha_2$ (with $\alpha_1 = 2\mu_1 <1$). Then, we get $z_{t} \le \alpha_1^2 z_{t-2} + (1+\alpha_1) \alpha_2$, and finally, $z_t \le \alpha_1^t z_0 + (1+\alpha_1+\ldots+\alpha_1^{t-1})\alpha_2 = \alpha_1^t z_0 + \frac{1-\alpha_1^t}{1-\alpha_1} \cdot \alpha_2 \le \alpha_1^t z_0 + \frac{1}{1-\alpha_1} \cdot \alpha_2 = \alpha_1^t z_0 + O(\alpha_2)$, where we use the assumption that $1-2\mu_1 = \Theta(1)$ in the last equality.}

\section{Proof of Theorem~\ref{thm:opt_known} (Optimal Solutions for Known Nonlinearity)}

Before proving the theorem, we present some auxiliary lemmas.

\subsection{Auxiliary Results for Theorem~\ref{thm:opt_known}}

First, from the mean value theorem (MVT), we have the following simple lemma.
\begin{lemma}\label{lem:simple_lu}
 For any $\bx_1,\bx_2 \in \bbR^p$, we have 
 \begin{equation}
  l \|\bA(\bx_1 -\bx_2)\|_2 \le \|f(\bA\bx_1)-f(\bA\bx_2)\|_2 \le u \|\bA(\bx_1 -\bx_2)\|_2.
 \end{equation}
\end{lemma}
\begin{proof}
Note that
 \begin{align}
 \|f(\bA \bx_1) - f(\bA \bx_2)\|_2^2 & = \sum_{i =1}^n \left(f(\ba_i^T\bx_1)-f(\ba_i^T\bx_2)\right)^2 \\
 & = \sum_{i =1}^n f'(\alpha_i)^2 \langle \ba_i, \bx_1 - \bx_2\rangle^2 \label{eq:firt_MVT}\\
 & \le u^2 \|\bA(\bx_1 -\bx_2)\|_2^2,
\end{align}
where we use the MVT in~\eqref{eq:firt_MVT} and $\alpha_i$ is in the interval between $\ba_i^T\bx_1$ and $\ba_i^T\bx_2$. Similarly, we have $\|f(\bA \bx_1) - f(\bA \bx_2)\|_2 \ge l\|\bA(\bx_1-\bx_2)\|_2$.
\end{proof}

Next, we present the following lemma for a standard concentration inequality for the sum of sub-Gaussian random variables. 
\begin{lemma} {\em (Hoeffding-type inequality \cite[Proposition~5.10]{vershynin2010introduction})}
\label{lem:large_dev_Gaussian} Let $X_{1}, \ldots , X_{N}$ be independent zero-mean sub-Gaussian random variables, and let $K = \max_i \|X_i\|_{\psi_2}$. Then, for any $\balpha=[\alpha_1,\alpha_2,\ldots,\alpha_N]^T \in \mathbb{R}^N$ and any $\epsilon\ge 0$, it holds that
\begin{equation}
\mathbb{P}\left( \Big|\sum_{i=1}^{N} \alpha_i X_{i}\Big| \ge \epsilon\right) \le   \exp\left(1-\frac{c \epsilon^2}{K^2\|\balpha\|_2^2}\right),
\end{equation}
where $c>0$ is a constant.
\end{lemma}

With the above lemmas in place, we are now ready to prove Theorem~\ref{thm:opt_known}.

\subsection{Proof of Theorem~\ref{thm:opt_known}}

Since $\hat{\bx}$ is a solution to~\eqref{eq:gen_lasso_known} and $\bar{\bx} \in \calK$, we have
 \begin{equation}
  \|f(\bA\hat{\bx}) - \tilde{\by}\|_2 \le \|f(\bA\bar{\bx}) - \tilde{\by}\|_2, 
 \end{equation}
which gives
\begin{align}
 &\frac{1}{n}\|f(\bA\hat{\bx}) - f(\bA\bar{\bx})\|_2^2 \le \frac{2}{n}\langle f(\bA\hat{\bx}) - f(\bA\bar{\bx}), \tilde{\by}-f(\bA\bar{\bx})\rangle \\
 & = \frac{2}{n}\langle f(\bA\hat{\bx}) - f(\bA\bar{\bx}), \tilde{\by}-\by\rangle + \frac{2}{n}\langle f(\bA\hat{\bx}) - f(\bA\bar{\bx}), \bm{\eta}\rangle + \frac{2}{n}\langle f(\bA\hat{\bx}) - f(\bA\bar{\bx}), f(\bA\bx^*)-f(\bA\bar{\bx})\rangle, \label{eq:three_terms_opt_known}
\end{align}
where we use $\bm{\eta} = \by-f(\bA\bx^*)$ in~\eqref{eq:three_terms_opt_known}. From the assumptions about JLE and bounded spectral norm ({\em cf.}~\eqref{eq:JLE_assump} and~\eqref{eq:BSN_assump}), as well as using a chaining argument similar to those in~\cite{bora2017compressed,liu2020sample}, we obtain that $\frac{1}{\sqrt{n}}\bA$ satisfies the TS-REC$(\calK,\alpha,\delta)$ for $\calK= G(B_2^k(r))$, $\alpha \in (0,1)$ and $\delta >0$. In particular, setting $\alpha  = 0.5$, we have 
\begin{equation}
 \frac{1}{\sqrt{n}}\left\|\bA (\hat{\bx}-\bar{\bx})\right\|_2 \ge \frac{1}{2}\|\hat{\bx}-\bar{\bx}\|_2 -\delta. 
\end{equation}
Taking the square on both sides, we obtain 
\begin{equation}\label{eq:non_unif2}
 \frac{1}{4}\|\hat{\bx}-\bar{\bx}\|_2^2 \le \frac{1}{n}\left\| \bA (\hat{\bx}-\bar{\bx})\right\|_2^2 + \delta \|\hat{\bx}-\bar{\bx}\|_2 \le \frac{1}{ n l^2}\|f(\bA\hat{\bx})-f(\bA\bar{\bx})\|_2^2 + \delta \|\hat{\bx}-\bar{\bx}\|_2. 
\end{equation}
In the following, we control the three terms in the right-hand-side of~\eqref{eq:three_terms_opt_known} separately. 
\begin{itemize}
\item Upper bounding $\big|\frac{2}{n}\langle f(\bA\hat{\bx}) - f(\bA\bar{\bx}), \tilde{\by}-\by\rangle\big|$: We have
\begin{align}
 \left|\frac{2}{n}\langle f(\bA\hat{\bx}) - f(\bA\bar{\bx}), \tilde{\by}-\by\rangle\right| &\le \left\|\frac{2}{\sqrt{n}}(f(\bA\hat{\bx}) - f(\bA\bar{\bx}))\right\|_2 \cdot \left\|\frac{1}{\sqrt{n}}(\by -\tilde{\by})\right\|_2 \\
 & \le \frac{2\tau u}{\sqrt{n}}\|\bA(\hat{\bx}-\bar{\bx})\|_2 \label{eq:advNoise_known_eq1}\\
 & \le 2\tau u \cdot O(\|\hat{\bx}-\bar{\bx}\|_2 + \delta),\label{eq:advNoise_known_eq2}
\end{align}
where we use Lemma~\ref{lem:simple_lu} in~\eqref{eq:advNoise_known_eq1} and use the TS-REC in~\eqref{eq:advNoise_known_eq2}. 

 \item Upper bounding $\big|\frac{2}{n}\langle f(\bA\hat{\bx}) - f(\bA\bar{\bx}), \bm{\eta}\rangle\big|$: We have
    \begin{align}
  \frac{1}{n}|\langle f(\bA\hat{\bx}) - f(\bA\bar{\bx}), \bm{\eta}\rangle| & = \left|\frac{1}{n}\sum_{i=1}^n \eta_i \left(f\left(\ba_i^T \hat{\bx}\right) - f\left(\ba_i^T \bar{\bx}\right)\right)\right| \\
  & = \left|\frac{1}{n}\sum_{i=1}^n \eta_i f'(\beta_i) \langle\ba_i, \hat{\bx}-\bar{\bx} \rangle\right|, \label{eq:sum_eta_f}
 \end{align}
where $\beta_i$ is in the interval between $\ba_i^T\hat{\bx}$ and $\ba_i^T\bar{\bx}$. From setting $\alpha = \frac{1}{2}$ in the TS-REC, we obtain
\begin{equation}
 \sqrt{\sum_{i=1}^n (f'(\beta_i)\langle\ba_i, \hat{\bx}-\bar{\bx} \rangle)^2} \le u\|\bA(\hat{\bx}-\bar{\bx})\|_2 \le u\sqrt{n}\left(\frac{3}{2}\|\hat{\bx}-\bar{\bx}\|_2 +\delta\right).\label{eq:TS-REC_eta_coef}
\end{equation}
Then, from Lemma~\ref{lem:large_dev_Gaussian} and the assumption that $\eta_1,\eta_2,\ldots,\eta_m$ are i.i.d.~zero-mean sub-Gaussian with $\sigma= \max_{i \in [n]}\|\eta_i\|_{\psi_2}$, for any $t>0$, we have with probability $1-e^{-\Omega(t)}$ that
\begin{equation}
 \left|\frac{1}{n}\sum_{i=1}^n \eta_i f'(\beta_i) \langle\ba_i, \hat{\bx}-\bar{\bx} \rangle\right| \le  \frac{C\sigma\sqrt{\sum_{i=1}^n (f'(\beta_i)\langle\ba_i, \hat{\bx}-\bar{\bx} \rangle)^2}}{n} \cdot \sqrt{t}.\label{eq:sub_Gaussian_concen_eta}
\end{equation}
Combining~\eqref{eq:sum_eta_f},~\eqref{eq:TS-REC_eta_coef},~\eqref{eq:sub_Gaussian_concen_eta}, and setting $t = k \log\frac{Lr}{\delta}$, we obtain with probability at least $1-e^{-\Omega(k \log \frac{Lr}{\delta})}$ that
\begin{equation}\label{eq:non_unif7}
 \frac{1}{n}|\langle \bmeta, f(\bA\hat{\bx}) - f(\bA\bar{\bx}) \rangle| \le O\left(u\sigma   \sqrt{\frac{k\log\frac{Lr}{\delta}}{n}}\right) (\|\hat{\bx}-\bar{\bx}\|_2 +\delta).
\end{equation}

\item Upper bounding $\big|\frac{2}{n}\langle f(\bA\hat{\bx}) - f(\bA\bar{\bx}), f(\bA\bx^*)-f(\bA\bar{\bx})\rangle\big|$: We have 
\begin{align}
 &\frac{1}{n}\left|\langle f(\bA\hat{\bx}) - f(\bA\bar{\bx}), f(\bA\bx^*)-f(\bA\bar{\bx})\rangle \right| \le \frac{1}{n} \|f(\bA\hat{\bx}) - f(\bA\bar{\bx})\|_2 \|f(\bA\bx^*)-f(\bA\bar{\bx})\|_2  \\
 & \le \frac{u^2}{n}\|\bA(\hat{\bx}-\bar{\bx})\|_2 \cdot \|\bA(\bx^*-\bar{\bx})\|_2 \label{eq:aaaa_exp1} \\
 & \le u^2 \left(\frac{3}{2}\|\hat{\bx}-\bar{\bx}\|_2 +\delta\right) \cdot \frac{3}{2}\|\bx^*-\bar{\bx}\|_2 \label{eq:aaaa_exp2}\\
 & = O\left(u^2 \|\bx^*-\bar{\bx}\|_2 \right) (\|\hat{\bx}-\bar{\bx}\|_2 +\delta)\label{eq:non_unif8},
\end{align}
where we use Lemma~\ref{lem:simple_lu} in~\eqref{eq:aaaa_exp1}, and in~\eqref{eq:aaaa_exp2}, we use the TS-REC with $\alpha =\frac{1}{2}$ for $\|\bA(\hat{\bx}-\bar{\bx})\|_2$ and the JLE ({\em cf.}~\eqref{eq:JLE_assump}) with $\epsilon = \frac{1}{2}$ for $\|\bA(\bx^*-\bar{\bx})\|_2$ (note that both $\bx^*$ and $\bar{\bx}$ are fixed vectors). 
\end{itemize}

Combining~\eqref{eq:three_terms_opt_known},~\eqref{eq:non_unif2},~\eqref{eq:advNoise_known_eq2},~\eqref{eq:non_unif7}, and~\eqref{eq:non_unif8}, and recalling that $l,u$ are fixed constants, we obtain 
\begin{align}
 \|\hat{\bx}-\bar{\bx}\|_2^2 \le O\left(\frac{\sigma \sqrt{k\log\frac{Lr}{\delta}}}{ \sqrt{n}}+\|\bx^*-\bar{\bx}\|_2 + \tau + \delta\right)(\|\bar{\bx}-\hat{\bx}\|_2 +\delta),\label{eq:non_unif_final}
\end{align}
which gives the desired result.\footnote{~\eqref{eq:non_unif_final} is of the form $x^2 \le \alpha_3(x+\alpha_4)$, where $x = \|\hat{\bx}-\bar{\bx}\|_2$, $\alpha_4 = \delta$ and $\alpha_4 = O(\alpha_3)$ (since the $\delta$ term also appears in $\alpha_3$). From the quadratic formula, we obtain $x \le \frac{\alpha_3 + \sqrt{\alpha_3^2 + 4 \alpha_3\alpha_4}}{2} = O(\alpha_3)$.}

\section{Proof of Theorem~\ref{thm:pgd} (PGD for Known Nonlinearity)}

Before proving the theorem, we prove some auxiliary lemmas.

\subsection{Auxiliary Results for Theorem~\ref{thm:pgd}}

First, we state the following useful lemma.

\begin{lemma}\label{lem:x1x2_JLE}
 Suppose that $\bA$ satisfies the assumption about JLE ({\em cf.}~\eqref{eq:JLE_assump}). Then, for any finite sets $E_1,E_2 \subseteq \bbR^p$ and $\epsilon \in (0,1)$, for all $\bx_1\in E_1$, $\bx_2 \in E_2$, we have 
 \begin{equation}
  \left|\frac{1}{n}\bx_1^T\bA^T\bA\bx_2 - \bx_1^T\bx_2\right| \le \epsilon\left(\|\bx_1\|_2^2+\|\bx_2\|_2^2\right).
 \end{equation}
\end{lemma}
\begin{proof}
 From the JLE, we have
 \begin{equation}
  \left(1-\frac{\epsilon}{2}\right)\|\bx_1 + \bx_2\|_2^2 \le \left\|\frac{1}{\sqrt{n}}\bA(\bx_1+\bx_2)\right\|_2^2 \le \left(1+\frac{\epsilon}{2}\right)\|\bx_1 + \bx_2\|_2^2,
 \end{equation}
and
\begin{equation}
  \left(1-\frac{\epsilon}{2}\right)\|\bx_1 - \bx_2\|_2^2 \le \left\|\frac{1}{\sqrt{n}}\bA(\bx_1-\bx_2)\right\|_2^2 \le \left(1+\frac{\epsilon}{2}\right)\|\bx_1 - \bx_2\|_2^2.
 \end{equation}
 Then, from
 \begin{equation}
  \frac{1}{n}\bx_1^T\bA^T\bA\bx_2 = \frac{\|\bA(\bx_1+\bx_2)\|_2^2 - \|\bA(\bx_1-\bx_2)\|_2^2}{4n},
 \end{equation}
 and 
 \begin{equation}
  \bx_1^T\bx_2 = \frac{\|\bx_1+\bx_2\|_2^2 - \|\bx_1-\bx_2\|_2^2}{4},
 \end{equation}
we obtain the desired bound. 
\end{proof}

Based on Lemma~\ref{lem:x1x2_JLE}, we present the following lemma.
\begin{lemma}\label{lem:imp_pgd}
 Let $\bar{\bx} = \arg\min_{\bx \in\calK}\|\bx^*-\bx\|_2$. For any $t\in\bbN$, $\delta >0$ and $\epsilon \in (0,1)$, we have
 \begin{align}
  &\left|\left\langle \bx^{(t)}-\bar{\bx} - \frac{\zeta}{n}\bA^T\left(\left(f(\bA\bx^{(t)})-f(\bA\bar{\bx})\odot f'(\bA\bx^{(t)})\right),\bx^{(t+1)}-\bar{\bx}\right)\right\rangle\right| \le \max \{1-\zeta l^2, \zeta u^2 -1\} \|\bx^{(t)}-\bar{\bx}\|_2 \|\bx^{(t+1)}-\bar{\bx}\|_2 \nonumber\\
  &  \indent + 
\zeta u^2 \epsilon \left(\|\bx^{(t)}-\bar{\bx}\|_2^2 + \|\bx^{(t+1)}-\bar{\bx}\|_2^2\right) + C \zeta u^2\delta(\|\bx^{(t)}-\bar{\bx}\|_2 + \|\bx^{(t+1)}-\bar{\bx}\|_2 +\delta).
 \end{align}
\end{lemma}
\begin{proof}
We have
 \begin{align}
&\bx^{(t)}-\bar{\bx} - \frac{\zeta}{n} \bA^T \left(\left(f(\bA\bx^{(t)}) - f(\bA\bar{\bx})\right)\odot f'(\bA\bx^{(t)})\right) \nonumber\\
& =\left(\bx^{(t)}-\bar{\bx}\right) -\frac{\zeta}{n} \sum_{i=1}^n f'(\beta_{t,i})f'(\ba_i^T\bx^{(t)}) \ba_i\ba_i^T (\bx^{(t)} -\bar{\bx}),\label{eq:imp_pgd_comb1}
\end{align}
where $\beta_{t,i}$ is in the interval between $\ba_i^T\bx^{(t)}$ and $\ba_i^T\bar{\bx}$. From Lemma~\ref{lem:x1x2_JLE}, the assumption about bounded spectral norm ({\em cf.}~\eqref{eq:BSN_assump}), a similar chaining argument to those in~\cite{bora2017compressed,liu2020sample}, and Appendix~\ref{app:proofThm1}, we obtain
\begin{align}
 &\left|\frac{1}{n} \left\langle\sum_{i=1}^n \ba_i\ba_i^T \left(\bx^{(t)} -\bar{\bx}\right), \bx^{(t+1)}-\bar{\bx} \right\rangle - \left(\bx^{(t+1)}-\bar{\bx}\right)^T\left(\bx^{(t)} -\bar{\bx}\right)\right|\nonumber\\ 
 &= \left|\frac{1}{n}\left(\bx^{(t+1)}-\bar{\bx}\right)^T\bA^T\bA\left(\bx^{(t)} -\bar{\bx}\right) - \left(\bx^{(t+1)}-\bar{\bx}\right)^T\left(\bx^{(t)} -\bar{\bx}\right)  \right|\\
 & \le \epsilon \left(\|\bx^{(t)}-\bar{\bx}\|_2^2 + \|\bx^{(t+1)}-\bar{\bx}\|_2^2\right) + C\delta(\|\bx^{(t)}-\bar{\bx}\|_2 + \|\bx^{(t+1)}-\bar{\bx}\|_2 +\delta).
\end{align}
Hence, we have
\begin{align}
 &\frac{\zeta}{n} \left\langle\sum_{i=1}^n f'(\beta_{t,i})f'(\ba_i^T\bx^{(t)}) \ba_i\ba_i^T (\bx^{(t)} -\bar{\bx}),\bx^{(t+1)}-\bar{\bx} \right\rangle  \nonumber\\
 &\le \zeta u^2 \left(\bx^{(t+1)}-\bar{\bx}\right)^T\left(\bx^{(t)} -\bar{\bx}\right) + \zeta u^2\epsilon \left(\|\bx^{(t)}-\bar{\bx}\|_2^2 + \|\bx^{(t+1)}-\bar{\bx}\|_2^2\right) + \zeta u^2 C\delta(\|\bx^{(t)}-\bar{\bx}\|_2 + \|\bx^{(t+1)}-\bar{\bx}\|_2 +\delta),\label{eq:imp_pgd_comb2}
\end{align}
and 
\begin{align}
 &\frac{\zeta}{n} \left\langle\sum_{i=1}^n f'(\beta_{t,i})f'(\ba_i^T\bx^{(t)}) \ba_i\ba_i^T (\bx^{(t)} -\bar{\bx}),\bx^{(t+1)}-\bar{\bx} \right\rangle  \nonumber\\
 &\ge \zeta l^2 \left(\bx^{(t+1)}-\bar{\bx}\right)^T\left(\bx^{(t)} -\bar{\bx}\right) - \zeta l^2\epsilon \left(\|\bx^{(t)}-\bar{\bx}\|_2^2 - \|\bx^{(t+1)}-\bar{\bx}\|_2^2\right) + \zeta l^2 C\delta(\|\bx^{(t)}-\bar{\bx}\|_2 + \|\bx^{(t+1)}-\bar{\bx}\|_2 +\delta).\label{eq:imp_pgd_comb3}
\end{align}
Combining~\eqref{eq:imp_pgd_comb1},~\eqref{eq:imp_pgd_comb2}, and~\eqref{eq:imp_pgd_comb3}, and using the inequality $\big|\bx_1^T\bx_2\big| \le \|\bx_1\|_2\|\bx_2\|_2$, we obtain the desired result.
\end{proof}

\subsection{Proof of Theorem~\ref{thm:pgd}}

For any $t \in \bbN$, let
\begin{equation}
 \tilde{\bx}^{(t+1)} = \bx^{(t)} - \frac{\zeta}{n} \bA^T \left((f(\bA\bx^{(t)})-\tilde{\by}) \odot f'(\bA\bx^{(t)})\right).
\end{equation}
From 
\begin{equation}
 \|\tilde{\bx}^{(t+1)} - \bx^{(t+1)}\|_2 \le \|\tilde{\bx}^{(t+1)} - \bar{\bx}\|_2,
\end{equation}
we have
\begin{align}
 \|\bx^{(t+1)}- \bar{\bx}\|_2^2  & \le 2\langle \tilde{\bx}^{(t+1)}- \bar{\bx},\bx^{(t+1)}- \bar{\bx}\rangle \label{eq:factor2} \\
 & = 2 \left\langle \bx^{(t)}- \bar{\bx} -\frac{\zeta}{n} \bA^T \left(\left(f(\bA\bx^{(t)})-\tilde{\by}\right)\odot f'(\bA\bx^{(t)}) \right), \bx^{(t+1)} - \bar{\bx}\right\rangle \\
 & = 2 \left\langle \bx^{(t)}-\bar{\bx} -\frac{\zeta}{n} \bA^T \left(\left(f(\bA\bx^{(t)})-f(\bA\bar{\bx})\right)\odot f'(\bA\bx^{(t)}) \right), \bx^{(t+1)} -\bar{\bx}\right\rangle \nonumber \\
 & \indent - \frac{2\zeta}{n}\cdot \left\langle \bA^T\left(\left(f(\bA\bar{\bx})-f(\bA\bx^*)\right) \odot f'(\bA\bx^{(t)})\right),\bx^{(t+1)} -\bar{\bx}\right\rangle \nonumber\\
 & \indent + \frac{2\zeta}{n}\cdot \left\langle \bA^T\left(\bmeta \odot f'(\bA\bx^{(t)})\right),\bx^{(t+1)} -\bar{\bx}\right\rangle \nonumber\\
 & \indent - \frac{2\zeta}{n}\cdot \left\langle \bA^T\left((\by-\tilde{\by}) \odot f'(\bA\bx^{(t)})\right),\bx^{(t+1)} -\bar{\bx}\right\rangle. \label{eq:pgd_1}
\end{align}
We control the four terms in~\eqref{eq:pgd_1} separately. 
\begin{itemize}
 \item The first term: Let $\tilde{\mu}_2 = \max \{1-\zeta l^2, \zeta u^2 -1\}$. From Lemma~\ref{lem:imp_pgd}, we obtain that the absolute value of the first term is upper bounded by
 \begin{align}
 &2\left|\left\langle \bx^{(t)}-\bar{\bx} - \frac{\zeta}{n}\bA^T\left(\left(f(\bA\bx^{(t)})-f(\bA\bar{\bx})\odot f'(\bA\bx^{(t)})\right),\bx^{(t+1)}-\bar{\bx}\right)\right\rangle\right| \le 2\tilde{\mu}_2 \|\bx^{(t)}-\bar{\bx}\|_2 \|\bx^{(t+1)}-\bar{\bx}\|_2 \nonumber\\
  &  \indent + 
2\zeta u^2 \epsilon \left(\|\bx^{(t)}-\bar{\bx}\|_2^2 + \|\bx^{(t+1)}-\bar{\bx}\|_2^2\right) + C \zeta u^2\delta(\|\bx^{(t)}-\bar{\bx}\|_2 + \|\bx^{(t+1)}-\bar{\bx}\|_2 +\delta).\label{eq:pgd_known_first}
 \end{align}

\item The second term: Similarly to~\eqref{eq:non_unif8}, we obtain 
\begin{equation}\label{eq:pgd_known_second}
 \left|\frac{2\zeta}{n}\cdot \left\langle \bA^T\left(\left(f(\bA\bar{\bx})-f(\bA\bx^*)\right) \odot f'(\bA\bx^{(t)})\right),\bx^{(t+1)} -\bar{\bx}\right\rangle\right| \le C\|\bar{\bx}-\bx^*\|_2 (\|\bx^{(t+1)} -\bar{\bx}\|_2 + \delta).
\end{equation}

\item The third term: Similarly to~\eqref{eq:non_unif7}, we obtain with probability $1-e^{-\Omega(k \log \frac{Lr}{\delta})}$ that
\begin{equation}\label{eq:pgd_known_third}
 \left|\frac{2\zeta}{n}\cdot \left\langle \bA^T\left(\bmeta \odot f'(\bA\bx^{(t)})\right),\bx^{(t+1)} -\bar{\bx}\right\rangle\right| \le O\left(\sigma\sqrt{\frac{k\log\frac{Lr}{\delta}}{n}}\right) (\|\bx^{(t+1)} -\bar{\bx}\|_2 + \delta).
\end{equation}

\item The fourth term: Similarly to~\eqref{eq:advNoise_known_eq2}, we obtain 
\begin{equation}\label{eq:pgd_known_fourth}
 \left|\frac{2\zeta}{n}\cdot \left\langle \bA^T\left((\by-\tilde{\by}) \odot f'(\bA\bx^{(t)})\right),\bx^{(t+1)} -\bar{\bx}\right\rangle\right| \le \tau\cdot O(\|\bx^{(t+1)} -\bar{\bx}\|_2 + \delta).
\end{equation}
\end{itemize}

Combining~\eqref{eq:pgd_1},~\eqref{eq:pgd_known_first},~\eqref{eq:pgd_known_second},~\eqref{eq:pgd_known_third} and~\eqref{eq:pgd_known_fourth}, we have that if $n = \Omega\big(\frac{k}{\epsilon^2}\log\frac{Lr}{\delta}\big)$, with probability $1-e^{-\Omega(n\epsilon^2)}$,
\begin{align}\label{eq:pgd_known_final}
 \|\bx^{(t+1)}-\bar{\bx}\|_2^2 & \le \left(2\tilde{\mu}_2 \|\bx^{(t)}-\bar{\bx}\|_2 + O\left(\sigma\sqrt{\frac{k\log\frac{Lr}{\delta}}{n}}+ \|\bx^*-\bar{\bx}\|_2+\tau+\delta\right)\right) \cdot (\|\bx^{(t+1)}-\bar{\bx}\|_2 +\delta) \nonumber \\
 & \indent +2\zeta u^2 \epsilon \left(\|\bx^{(t)}-\bar{\bx}\|_2^2 + \|\bx^{(t+1)}-\bar{\bx}\|_2^2\right).
\end{align}
Considering $\zeta,l,u$ as fixed positive constants, and selecting $\epsilon >0 $ to be sufficiently small, we obtain 
\begin{equation}
 \|\bx^{(t+1)}-\bar{\bx}\|_2 \le 2\mu_2 \|\bx^{(t)}-\bar{\bx}\|_2 + O\left(\sigma\sqrt{\frac{k\log\frac{Lr}{\delta}}{n}}+ \|\bx^*-\bar{\bx}\|_2+\tau+\delta\right).
\end{equation}

\section{Experimental Results for Noisy $1$-bit Measurement Models}

In this section, we present numerical results for the case when $f(x) = \mathrm{sign}(x+e)$ with $e \sim \calN(0,\sigma^2)$. Since this $f$ is not differentiable, we only perform Algorithm~\ref{algo:pgd_gLasso} (\texttt{PGD-G}), and do not perform Algorithm~\ref{algo:pgd_known} (\texttt{PGD-N}). We follow the SIM~\eqref{eq:sim_first} to generate the observations. For the MNIST dataset, we set $\sigma = 0.1$, and for the CelebA dataset, we set $\sigma = 0.01$. All other settings are the same as those in the main document. In Figs.~\ref{fig:mnist_imgs_sign} and~\ref{fig:celebA_imgs_sign}, we provide some examples of recovered images of the MNIST and CelebA datasets respectively. In particular, we compare with the Binary Iterative Hard Thresholding (\texttt{BIHT}) algorithm, which is a popular approach for $1$-bit CS with sparsity assumptions, and the~\texttt{1b-VAE} algorithm~\cite{liu2020sample}, which is a state-of-the-art method for $1$-bit CS with generative priors. From Figs.~\ref{fig:mnist_imgs_sign},~\ref{fig:celebA_imgs_sign} and~\ref{fig:quant_sign},  we observe that the sparsity-based method~\texttt{BIHT} leads to poor reconstruction. From these figures, we also observe that~\texttt{PGD-G} outperforms~\texttt{1b-VAE} on both the MNIST and CelebA datasets. Moreover, from Fig.~\ref{fig:quant_sign}, we see that the advantage of~\texttt{PGD-G} over~\texttt{1b-VAE} is more significant for the MNIST dataset with $\sigma = 0.1$, and is less so for the CelebA dataset with $\sigma = 0.01$. These experimental results reveal that for $1$-bit CS with generative priors,~\texttt{PGD-G} can be used to replace the current state-of-the-art method~\texttt{1b-VAE}, though~\texttt{PGD-G} does not make use of the knowledge of $f$, and is not specifically designed for noisy $1$-bit measurement models.

\begin{figure}[t]
\begin{center}
\includegraphics[width=1.0\columnwidth]{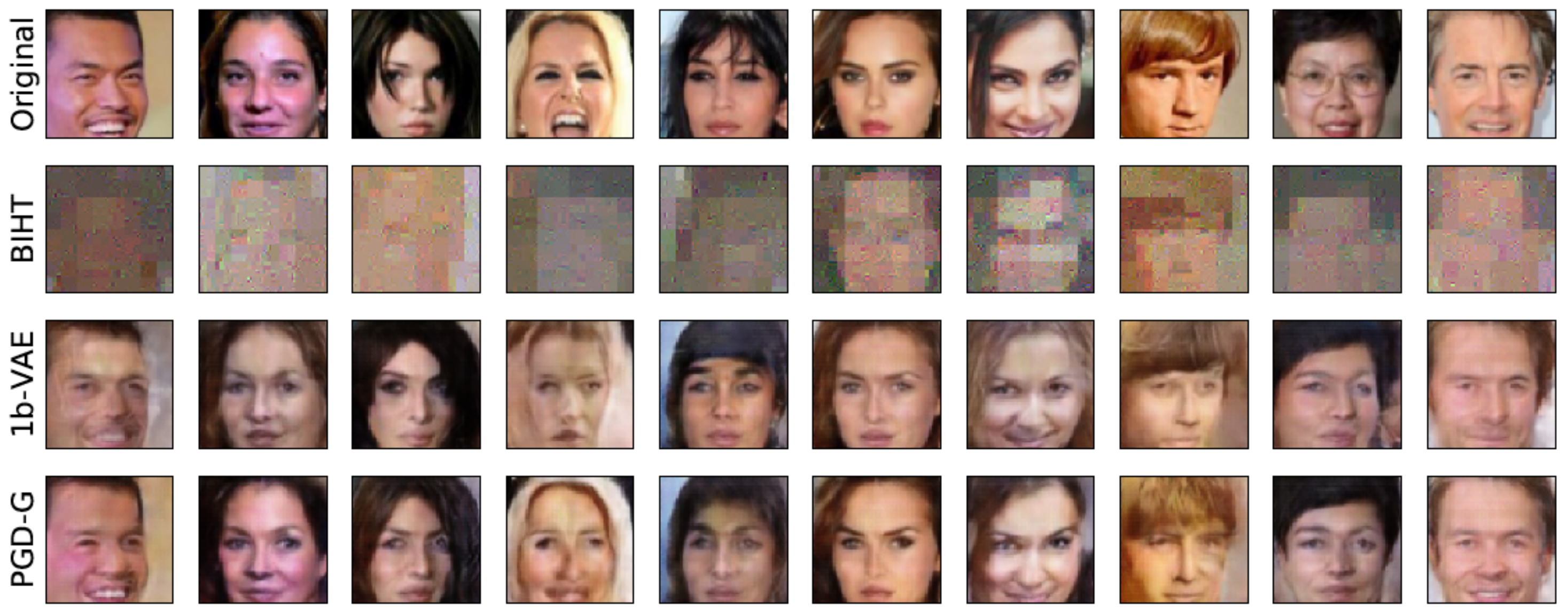}
\caption{Examples of reconstructed images of the CelebA dataset with $n = 1000$ measurements and $p = 12288$ dimensional vectors.}\label{fig:celebA_imgs_sign}   
\end{center}
\end{figure}

\begin{figure}[t]
\begin{center}
\begin{tabular}{cc}
\includegraphics[height=0.38\textwidth]{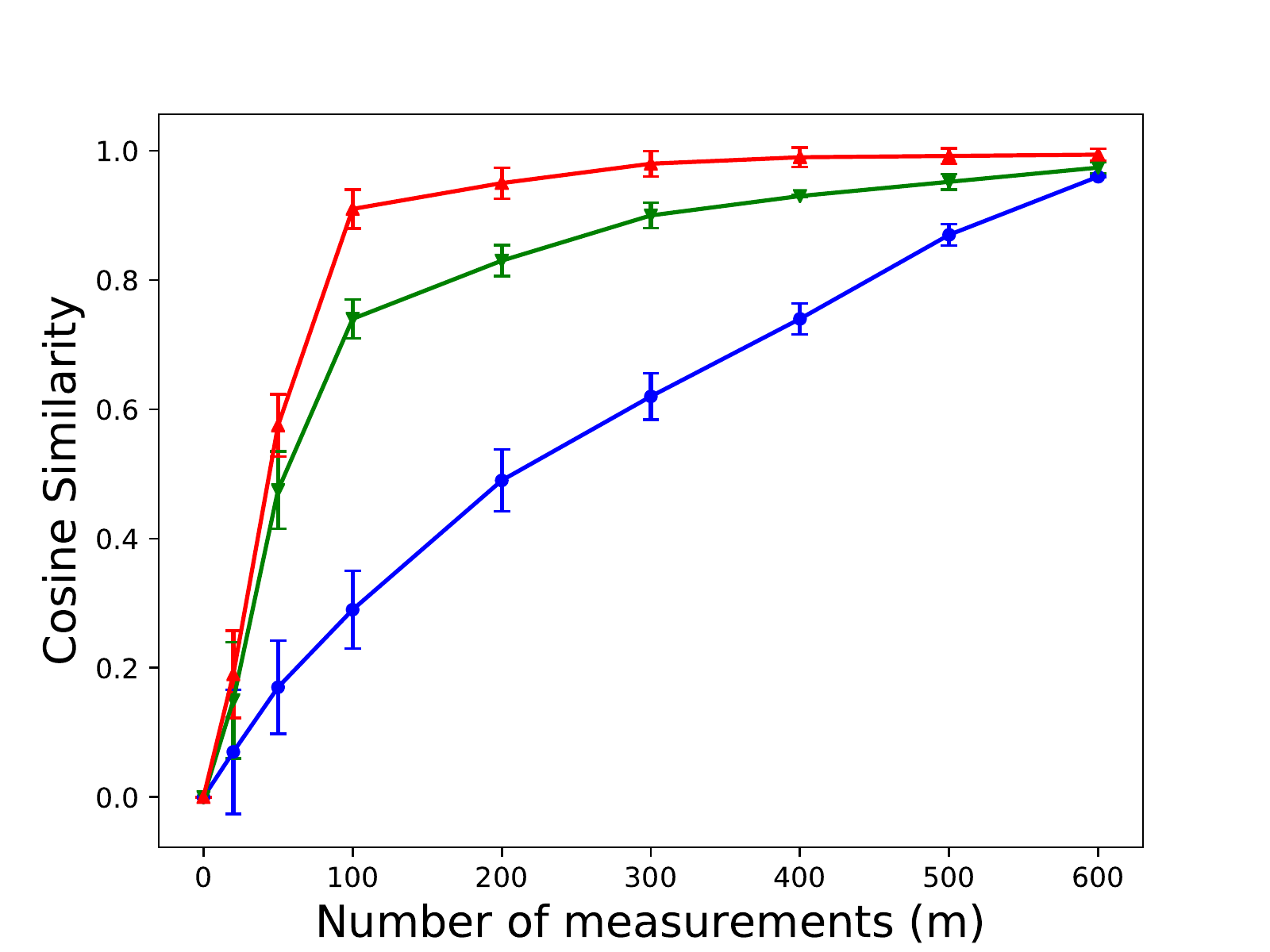} & \hspace{-0.5cm}
\includegraphics[height=0.38\textwidth]{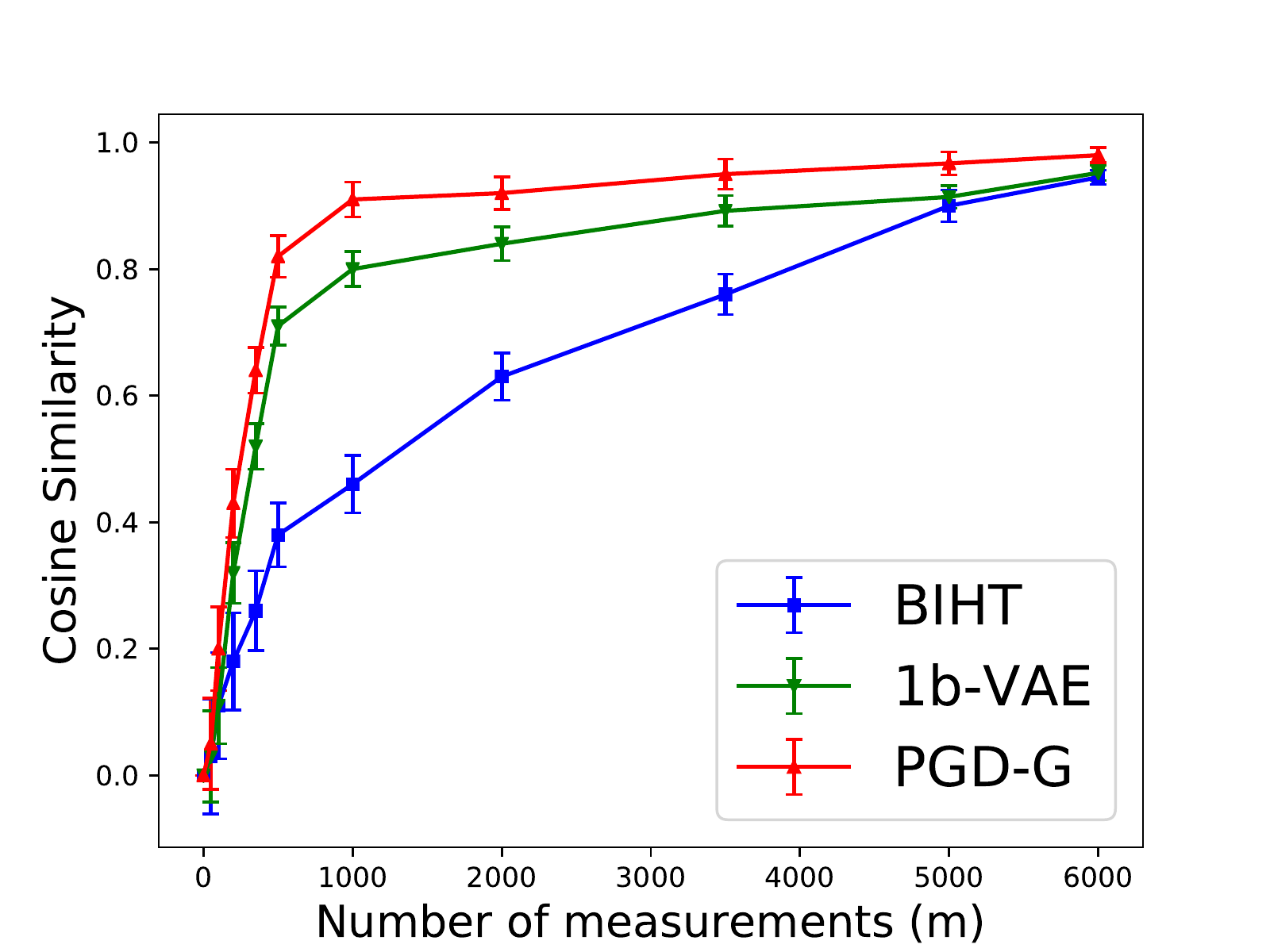} \\
{\small (a) Results on MNIST} & {\small (b) Results on CelebA}
\end{tabular}
\caption{Quantitative comparisons for the performance of~\texttt{1b-VAE} and~\texttt{PGD-G} according to Cosine Similarity.} \label{fig:quant_sign}
\end{center}
\end{figure} 

\bibliographystyle{IEEEtran}
\bibliography{writeups}

\end{document}